\documentclass[11pt]{article}

\usepackage[final]{acl}
\usepackage{algorithm}
\usepackage{algpseudocode}
\usepackage{times}
\usepackage{latexsym}
\usepackage{amssymb}
\usepackage{amsmath}
\usepackage{amsthm}
\usepackage[T1]{fontenc}
\usepackage{enumitem}
\newtheorem{theorem}{Theorem}
\newtheorem{lemma}{Lemma}

\newtheorem{assumption}{Assumption}

\newtheorem{remark}{Remark}
\newtheorem*{lemmafull}{Lemma~\ref{lem:implicit_patch}}
\DeclareMathOperator*{\argmin}{arg\,min}
\usepackage{soul}
\usepackage{stfloats}
\usepackage{subcaption}

\usepackage[utf8]{inputenc}

\usepackage{microtype}

\usepackage{inconsolata}

\usepackage{graphicx}

\usepackage{booktabs}
\usepackage{multirow}
\usepackage{adjustbox}
\usepackage{colortbl}
\usepackage{listings}
\usepackage{xcolor}
\usepackage[most]{tcolorbox}

\newtcolorbox{promptbox}[1]{
    float*=t,
    enhanced,
    colback=white,
    colframe=blue!70!black,
    coltitle=white,
    fonttitle=\bfseries,
    title=#1,
    boxrule=0.8pt,
    arc=3pt,
    left=6pt,
    right=6pt,
    top=4pt,
    bottom=4pt,
    breakable
}

\title{Graph-based Target Back-Propagation for Context Adaptation in Multi-LLM Agentic Systems}

\author{
 \textbf{Tan Zhu},
 \textbf{Tong Yao},
 \textbf{Kananart Kuwaranancharoen},
 \textbf{Amit Singh},
  \textbf{Yushang Lai},
 \\
 \textbf{Deepa Mohan},
 \textbf{Shankar Bhargava}
\\
Retail Intelligence, Walmart Global Tech
\\
Sunnyvale, CA, USA
\\
 \small{
   \{\href{mailto:tan.zhu@walmart.com}{tan.zhu}, 
   \href{mailto:tong.yao@walmart.com}{tong.yao}, 
   \href{mailto:Kananart.Kuwaranancharoen@walmart.com}{kananart.kuwaranancharoen}, 
   \href{mailto:Amit.Singh2@walmart.com}{amit.singh2},
    \href{mailto:Yushang.Lai@walmart.com}{yushang.lai}}\\
 \small{    
   \href{mailto:deepa.mohan@walmart.com}{deepa.mohan}, 
   \href{mailto:Shankara.bhargava@walmart.com}{shankara.bhargava}\}
   @walmart.com
 }
}

\begin{document}
\maketitle
\begin{abstract}

Context adaptation automates prompt engineering in LLM-based systems by iteratively revising tunable prompts from task feedback, without modifying model weights. Extending this paradigm to multi-LLM agentic systems is crucial: existing methods suffer from inaccurate credit assignment and lack convergence guarantees. We propose \textbf{G}raph-based \textbf{T}arget \textbf{B}ack-\textbf{P}ropagation (GTBP), a context adaptation framework for agentic workflows modeled as directed acyclic graphs. 
GTBP propagates local target outputs backward through the workflow graph and uses target--output discrepancies to guide a stage-wise prompt update mechanism.
Theoretically, we show that GTBP's stage-wise prompt updates become stable over iterations, and that a sufficiently capable LLM optimizer can decrease the overall objective.
Empirically, GTBP consistently outperforms strong baselines across three benchmarks while maintaining comparable computational cost.
\end{abstract}
\section{Introduction}

With input contexts such as task instructions, tool-use guidelines, and memory-retrieval
instructions that guide how external information is selected and incorporated \citep{schulhoff2024prompt},
large language models (LLMs) can coordinate multiple LLM-powered modules in
\emph{agentic systems} for multi-step problem solving
\citep{li2025review}. This paper studies \textbf{\emph{context adaptation}} for such
multi-module LLM systems, where input contexts are automatically revised from
task feedback while model weights remain frozen.
Let $\Phi$ denote an agentic system with frozen model weights and a tunable
prompt set $\Pi_{\Phi}$. Given a task distribution $\mathcal{T}$ and a
task-level loss function $\ell$, context adaptation aims to find the optimal $\Pi_{\Phi}^*$ minimizing $\ell$:
\begin{equation}
\Pi_{\Phi}^{*}
=
\argmin_{\Pi_{\Phi}}
\mathbb{E}_{x\sim\mathcal{T}}
\left[
\ell\!\left(
\Phi(x; \Pi_{\Phi})
\right)
\right].
\label{eq:intro_prompt_optimization}
\end{equation}


Early agentic systems often rely on human-designed contexts, prompts, and instructions to coordinate multiple LLM-powered modules
\citep{hong2024metagpt,qian2024chatdev}. Context adaptation automates this
process by updating tunable prompts from task feedback, without modifying the model weights. 
Recent advances in LLMs, including
longer context windows and lower-latency inference
\citep{gemini2024gemini15,openai2024gpt4o}, make it increasingly feasible to
treat LLMs not only as task solvers but also as optimizers that iteratively
inspect, critique, and revise prompts. This makes
context adaptation a practical alternative to costly manual prompt engineering,
especially for workflows in multi-module agentic systems.

In multi-module agentic systems, context adaptation introduces
\textbf{\emph{credit assignment}}: final errors must be attributed to the
responsible module-specific contexts \citep{meulemans2021credit}. 
Agentic workflows involve text-based contexts, discrete actions, and
black-box LLM modules, making gradient-based credit assignment, widely used in artificial neural networks (ANNs), difficult to
apply directly. 
To assign credit, methods such as GEPA and ACE employ LLM calls that take the entire agentic workflow as input
\citep{agrawal2025gepa,zhang2025ace}. However, because this assignment is
performed implicitly rather than through an explicitly defined
inference procedure, it can be non-deterministic and unclear which prompts should be
updated. This leads to the challenge known as \emph{attribution ambiguity}
\citep{huang2026textresnet}.

To explicitly define credit-assignment pipelines over multi-agent systems,
methods such as TextGrad \citep{yuksekgonul2024textgrad} and Agentic Neural
Networks \citep{ma2025agentic} model multi-agent systems with
graph structures and propagate LLM-generated textual feedback
across agent nodes. However, these pipelines remain largely heuristic and are
not directly coupled with the prompt optimization objective in
Eq.~\ref{eq:intro_prompt_optimization}. Moreover, after credits are assigned, 
heuristic rule-based LLM optimizer are employed to update prompts. Their convergence properties in multi-module agentic systems remain theoretically underexplored.


To address these limitations, we draw inspiration from Difference Target Propagation (DTP), which assigns credit by propagating local targets rather than
gradients \citep{lee2015difference,meulemans2020theoretical}. Building on this
idea, we propose \textbf{G}raph-based \textbf{T}arget
\textbf{B}ack-\textbf{P}ropagation (\textbf{GTBP}), a context adaptation pipeline for agentic systems.
In summary, this paper makes the following contributions:

\begin{itemize}[leftmargin=1.2em, nosep]
    \item We propose GTBP, a graph-based credit-assignment framework for
context adaptation in multi-LLM agentic systems. Instead of relying on
trajectory-level reflection, as shown in Fig. \ref{fig:overview of gtbp}, GTBP follows the workflow graph to infer local
target outputs for sub-modules and uses target-output
discrepancies to guide stage-wise prompt updates.
    
    \item We provide a theoretical analysis of GTBP, showing that its stage-wise prompt updates become stable over iterations and that the task-level objective defined in Eq.~\ref{eq:intro_prompt_optimization} can be decreased with a sufficiently capable LLM optimizer.



    \item We empirically validate GTBP on real-world benchmark datasets and show
that it consistently improves prompt optimization performance compared with
strong baseline methods.
\end{itemize}

\begin{figure*}
    \centering
    \includegraphics[width=0.9\textwidth]{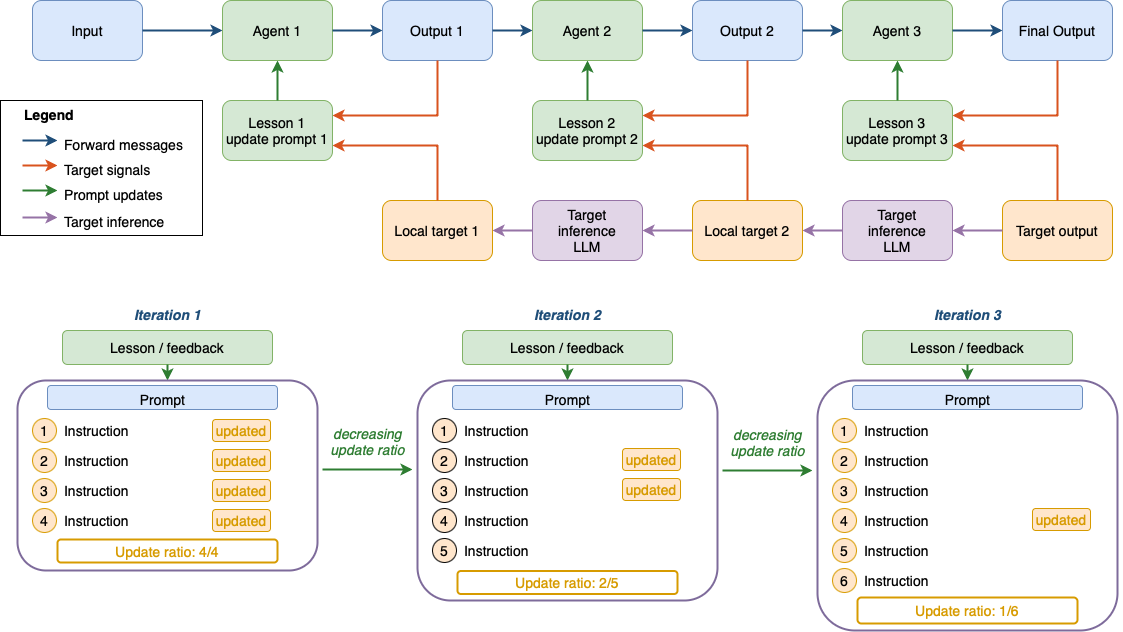}
   \caption{
Overview of GTBP for context adaptation. After forward execution, GTBP infers local targets for intermediate modules from the final target output. Each module then compares its current output with the inferred target, and the resulting credit signal is used to update its prompt without requiring gradient access.
}
    \label{fig:overview of gtbp}
\end{figure*}
\section{Related Work}
\subsection{Credit Assignment}


Credit assignment is a fundamental problem in learning systems, where global
feedback must be attributed to the internal components responsible for an
outcome \citep{minsky1961steps}. In ANNs, it is commonly
addressed through chain-rule-based gradient backpropagation
\citep{bottou2018optimization}. Alternative methods explore different credit
assignment mechanisms beyond gradient backpropagation. For example, Decoupled
Neural Interfaces decouple credit assignment by learning synthetic gradients for
intermediate modules \citep{jaderberg2017decoupled}. Difference Target
Propagation (DTP), in contrast, offers a target-based alternative by assigning
desired outputs to intermediate components \citep{bengio2014autoencoders,lee2015difference,meulemans2020theoretical}.

With rapid advances in LLM context length \citep{wang2024beyond} and inference efficiency
\citep{zhou2024efficient}, methods such as MIPROv2 \cite{miprov2}, GEPA
\citep{agrawal2025gepa}, and ACE \citep{zhang2025ace} use LLM reasoning
to approximate credit assignment over textual execution traces. However, because credit is assigned through black-box LLM-based reflection, the process remains difficult to control. To make it more structured, recent work such as TextGrad \citep{yuksekgonul2024textgrad} and Agentic Neural Networks \citep{ma2025agentic} propagates textual gradients through neural-network-like agentic workflows. TextResNet instead improves attribution through a credit-assignment-friendly residual architecture for signal decomposition and routing \citep{huang2026textresnet}.

\subsection{Context Adaptation}
Context adaptation extends manual prompt engineering \citep{schulhoff2024prompt} by automatically revising input contexts, such as prompts, instructions, or retrieved evidence, using task feedback or LLM-generated feedback \citep{sahoo2024systematic}. Representative methods differ in how they update textual contexts. 
Early automatic prompt optimization methods rely on non-textual feedback, such as task-level scores. APE \citep{zhou2022large} and OPRO \citep{yang2023large} search over LLM-generated prompt candidates using scalar evaluation signals. Moving beyond metric-based signals, recent studies such as Reflexion \citep{shinn2024reflexion} and GEPA \citep{agrawal2025gepa} exploit textual feedback for iterative improvement. With textual credit signals, prompts can be updated in more structured ways, such as playbooks in ACE \citep{zhang2025ace}.

The theoretical understanding of context adaptation is still emerging. Since modern LLMs are largely built on Transformer architectures, early work relates context adaptation to weight updates, showing that prompt updating can induce effects equivalent to parameter updates \citep{li2024closeness}. Recent work further suggests that context-conditioned behavior may be related to implicit optimization over model weights \citep{dherin2025learning}, although a general theoretical characterization of this relationship remains open.


\section{Problem Formulation}
\label{sec:problem_formulation}

In this section, we instantiate the objective in
Eq.~\ref{eq:intro_prompt_optimization} by defining the agentic system as a
directed graph, and specifying the workflow studied in this paper.

An agentic workflow is modeled as a directed acyclic graph
$\mathcal{G} = (\mathcal{V}, \mathcal{E})$, 
where each node $v \in \mathcal{V}$ corresponds to a LLM-invoking
submodule 
\begin{equation*}
    O_v = M_v(I_v; \pi_v)
\end{equation*}
parameterized by a tunable prompt $\pi_v$ and input $I_v$. 
For any node $v$, the input is formed by aggregating
the outputs of its predecessor nodes 
\begin{equation*}
    I_v = \{\, O_u \mid (u,v) \in \mathcal{E} \,\}.
\end{equation*}
In this graph-based workflow, node outputs are passed to downstream nodes and
aggregated to form subsequent inputs. Algorithm~\ref{alg:graph_forward} describes how the workflow executes the graph and combines intermediate nodes' outputs into the final output.


\begin{algorithm}[t]
\caption{Forward Execution of a Graph-based Agentic Workflow}
\label{alg:graph_forward}
\small
\begin{algorithmic}[1]
\Require Graph $\mathcal{G}=(\mathcal{V},\mathcal{E})$ with input-node set
$\mathcal{V}_{\mathrm{in}}$ and output node $v_{\mathrm{out}}$, external input
$x$, prompts $\{\pi_v\}_{v\in\mathcal{V}}$
\Ensure Final system output $y$

\For{each node $v \in \mathcal{V}$ in topological order}
    \If{$v \in \mathcal{V}_{\mathrm{in}}$}
        \State $I_v \gets x_v$
    \Else
        \State $I_v \gets \{\, O_u \mid (u,v)\in\mathcal{E} \,\}$
    \EndIf
    \State $O_v \gets M_v(I_v;\pi_v)$
\EndFor

\State $y \gets O_{v_{\mathrm{out}}}$
\end{algorithmic}
\end{algorithm}
In this paper, we focus on a one-hidden-layer agentic workflow. Given an input
message $x$, the workflow contains $K$ prompt-tunable LLM submodules that
produce intermediate responses
\begin{equation}
\label{first_layer}
O_k = M_k(x;\pi_k), \qquad k=1,\dots,K.
\end{equation}
Given these intermediate responses, the workflow contains an output node that aggregates them with the original input to produce the final output:
\begin{equation}
\label{second_layer}
y
=
M_{\mathrm{out}}(x,O_1,\dots,O_K;\pi_{\mathrm{out}}).
\end{equation}

Throughout this paper, all LLM parameters are fixed, and only the prompts
$\{\pi_1,\dots,\pi_K,\pi_{\mathrm{out}}\}$ are updated. This structure creates a
simple credit-assignment setting where the final error may arise from any
intermediate module or from the output module. 
Given this workflow graph, let
$\Pi_{\Phi} = \{\pi_v\}_{v\in\mathcal{V}}$ denote the prompt set of the overall
system. For a textual input $x$, graph execution induces the final system output
$y = \Phi(x;\Pi_{\Phi})$ following Algorithm~\ref{alg:graph_forward}. For an
input--reference pair $(x,y^*)$ sampled from the task distribution
$\mathcal{T}$, let $\ell(y,y^*)$ denote an output-level loss measuring the
discrepancy between the system output and the reference. The context adaptation
objective is then
\begin{equation}
\Pi_{\Phi}^{*}
=
\argmin_{\Pi_{\Phi}}
\mathbb{E}_{(x,y^*)\sim\mathcal{T}}
\left[
\ell\!\left(
\Phi(x;\Pi_{\Phi}), y^*
\right)
\right].
\label{eq:prompt_optimization}
\end{equation}
In our experiments, $\ell$ is instantiated using the evaluation metric
associated with each dataset $\mathcal{T}$.

\section{Proposed Method}
This section presents \textbf{\underline{G}}raph-based
\textbf{\underline{T}}arget \textbf{\underline{B}}ack-\textbf{\underline{P}}ropagation
(\textbf{GTBP}), which uses LLM-guided target propagation to trace
output-level feedback backward through the agentic workflow and assign
localized optimization signals to prompt-tunable sub-modules. We then introduce
a stage-wise mechanism for updating the corresponding prompts. An overview of GTBP is included in Figure~\ref{fig:overview of gtbp}.


\subsection{LLM-Guided Target Propagation}
\begin{algorithm}[t]
\caption{Target Propagation in GTBP}
\label{alg:gtbp_backward}
\small
\begin{algorithmic}[1]
\Require Graph $\mathcal{G}=(\mathcal{V},\mathcal{E})$, cached node outputs
$\{O_v\}_{v\in\mathcal{V}}$, output target
$\hat{O}_{v_{\mathrm{out}}}$, local contexts $\{c_v\}_{v\in\mathcal{V}}$
\Ensure Local targets $\{\hat{O}_v\}_{v\in\mathcal{V}}$

\For{each node $v \in \mathcal{V}$ in reverse topological order}
    \State Compute $\hat{I}_{v}$ using Eq.~\ref{eq:llm_backward_inference}
    \For{each node $u$ such that $(u,v)\in\mathcal{E}$}
        \State Get $\hat{O}_{u}$ from $\hat{I}_{v}$
    \EndFor
\EndFor

\State \Return $\{\hat{O}_v\}_{v\in\mathcal{V}}$
\end{algorithmic}
\end{algorithm}
Based on the graph-based workflow, GTBP adapts the idea of target
propagation to assign credit to LLM-based modules in agentic workflows. Given an input--reference pair
$(x,y^*)\in\mathcal{T}$, the graph $\mathcal{G} = (\mathcal{V}, \mathcal{E})$ produces node outputs
$\{O_v\}_{v\in\mathcal{V}}$ including the final output $O_{v_{\mathrm{out}}}$ following Algorithm \ref{alg:graph_forward}.
Let the expected output for each node $v\in\mathcal{V}$ be $\hat{O}_v$. Target propagation starts by assigning the reference output $y^*$ as the target of the
final output $\hat{O}_{v_{\mathrm{out}}}$.


The goal is then to propagate this output-level target backward through the
graph and infer local targets for the other nodes. For a node \(u\), its output
\(O_u\) may be consumed by multiple successor nodes. For each successor node
\(v\) with \((u,v)\in\mathcal{E}\), GTBP infers a target input for node \(v\)
by solving
\begin{equation}
\label{eq:target_message_optimization}
\argmin_{I}\left[\ell\left(M_v\!\left(I;\pi_v\right),\hat{O}_v\right)+\lambda \ell(I,I_v)\right],
\end{equation}
where \(I_v\) denotes the original input to node \(v\), and \(\hat I_v\)
denotes the inferred target input for node \(v\). The first term encourages
node \(v\) to produce target outputs for its predecessor nodes $u$ with $(u,v)\in\mathcal{E}$, while the second term
regularizes the inferred target input to remain close to the original input
\(I_v\). The coefficient \(\lambda>0\) controls this locality constraint.
\begin{remark}
In a general directed graph, a node may branch to multiple successors, which can lead to multiple backward-inferred targets for the same node. This ambiguity does not occur in our one-hidden-layer workflow. As defined in Eqs.~\ref{first_layer} and~\ref{second_layer}, each intermediate node connects only to the output node, and therefore receives a single target.
\end{remark}


However, Eq.~\ref{eq:target_message_optimization} is generally not analytically solvable
in LLM-based agentic systems, since sub-module $M_v$ might be a black-box LLM
call, and inputs and outputs within the system are text-based and non-differentiable. To address this problem, GTBP
approximates the target output using an LLM-guided backward
operator:
\begin{equation}
\label{eq:llm_backward_inference}
\hat{I}_v
=
\mathrm{LLM}\!\left(
P_{\mathrm{back}}(\hat{O}_v, I_v, c_v)
\right),
\end{equation}
where $P_{\mathrm{back}}(\cdot)$ is a backward inference prompt template and
$c_v$ denotes the local context of node $v$, including its role description,
system prompt, and task constraints. A concrete example of such a backward inference prompt is shown in Fig. \ref{fig:backward_prompt_template}.
\begin{figure}[t]
\centering
\small
\setlength{\fboxsep}{5pt}
\fbox{
\begin{minipage}{0.94\linewidth}
\textbf{Prompt template $P_{\mathrm{back}}$}

\vspace{2pt}
\textbf{Input:}
system prompt of node $v$, $c_v$;
original input of node $v$, $I_v$;
target output of node $v$, $\hat{O}_v$.

\vspace{2pt}
\textbf{Instruction:}
Generate a target input $\hat I_v$ for node $v$ such that node $v$ is more
likely to produce the target output $\hat O_v$ when processing $\hat I_v$ with
its current prompt. The generated target input should remain close in meaning
and structure to the original input $I_v$. 
\end{minipage}
}
\caption{Backward target-inference prompt template used to infer the target
input for a downstream node.}
\label{fig:backward_prompt_template}
\end{figure}
As shown in Algorithm \ref{alg:gtbp_backward}, by applying Eq.~\ref{eq:llm_backward_inference} from the output node to earlier
nodes in reverse topological order, GTBP generates targets for each sub-module node in the agentic workflow. 
\begin{algorithm}
\caption{Prompt Update for Tunable Prompts}
\label{alg:prompt_update}
\small
\begin{algorithmic}[1]
\Require Node $v$, current prompt
$\pi_v=\langle c_{v,1},\dots,c_{v,m}\rangle$,
module-level update set
$\mathcal{S}_v=\{(I_v^{(i)},O_v^{(i)},\hat{O}_v^{(i)})\}_{i=1}^{N}$,
update budget $n$, claim-addition marker $a_v\in\{0,1\}$
\Ensure Updated prompt $\pi_v$

\State Compute discrepancies
$\{\Delta(O_v^{(i)},\hat{O}_v^{(i)})\}_{i=1}^{N}$

\If{$a_v=1$}
    \State Add one new prompt claim to $\pi_v$
    \State Set $n \leftarrow n-1$
\EndIf

\State Select at most $n$ existing prompt entries from
$\langle c_{v,1},\dots,c_{v,m}\rangle$ for revision

\State Update the selected prompt entries in $\pi_v$ using Eq.~\ref{eq:prompt_optimizer} with
$\mathcal{S}_v$ and discrepancies
$a_v$

\State \Return Updated $\pi_v$
\end{algorithmic}
\end{algorithm}
\vspace*{-4mm}
\subsection{Stage-wise Prompt Updating}
Before introducing the proposed prompt update pipeline, we first specify the structure of the tunable prompt. The prompt associated with LLM module $v$ in the agentic workflow is formatted as an
ordered list of textual claims $\pi_v=\langle c_{v,1}, \dots, c_{v,m} \rangle$,
where $m$ denotes the number of prompt claims.

For each $(x,y^*)\in\mathcal{T}$, the target output $\hat{O}_v$ for each
sub-module $v$ in the agentic workflow $\mathcal{G}$ can be obtained by the
target propagation procedure described in Algorithm~\ref{alg:gtbp_backward}.
Given the input $I_v$ and the target output $\hat{O}_v$, the sub-module-level
prompt update for node $v$ can be expressed as
\begin{equation}
\pi_v^{\mathrm{opt}}
:=
\argmin_{\pi_v}
\mathbb{E}_{(x,y^*)\sim\mathcal{T}} \;
\ell\!\left(
M_v(I_v;\pi_v), \hat{O}_v
\right).
\label{eq:prompt_optimization_local_target}
\end{equation}
In practice, we approximate this objective using an LLM-based optimizer on a
mini-batch sampled from the task distribution:
\begin{equation}
\label{eq:prompt_optimizer}
\pi_v^{\mathrm{LLM}}
:=
\mathrm{LLM}\!\left(
P^{m,n}_{\mathrm{optimize}}(\pi_{v},\mathcal{S}_v(\mathcal{B}))
\right),
\end{equation}
where $\mathcal{B}\subseteq\mathcal{T}$ denotes a finite mini-batch sampled from
the task distribution, 
$\mathcal{S}_v(\mathcal{B})
=
\{(I_v^{(i)},O_v^{(i)},\hat{O}_v^{(i)})\}_{i\in\mathcal{B}}$
is the module-level update set for node $v$, and
$P_{\mathrm{optimize}}$ is the prompt optimization template for the LLM.
\begin{figure}[t]
\centering
\small
\setlength{\fboxsep}{5pt}
\fbox{
\begin{minipage}{0.94\linewidth}
\textbf{Prompt template $P_{\mathrm{optimize}}$}

\vspace{2pt}
\textbf{Input:}
current prompt $\pi_v$; local update set $\mathcal{S}_v(\mathcal{B})$;
edit budget $n$; claim-addition marker $a_v\in\{0,1\}$.

\vspace{2pt}
\textbf{Instruction:}
Given the examples in $\mathcal{S}_v(\mathcal{B})$, revise the current prompt
$\pi_v$ so that module $v$ is more likely to produce outputs close to the
target outputs $\hat{O}_v^{(i)}$ for inputs $I_v^{(i)}$.
Use the current outputs $O_v^{(i)}$ to identify the gap between current and
desired behavior. If $a_v=1$, add one new claim to the prompt and use one unit
of the edit budget. Then update no more than $n-a_v$ existing claims.
\end{minipage}
}
\caption{Prompt optimization template used by the stage-wise LLM optimizer.}
\label{fig:prompt_optimize_template}
\end{figure}
\subsection{Graph-based Target Back-Propogation}
\label{sec:llm_guided_target_propagation}

As shown in Algorithm \ref{alg:gtbp_minibatch}, GTBP performs context adaptation by repeatedly applying LLM-guided target propagation to mini-batches sampled from the task distribution and updating the tunable prompts accordingly. At each
iteration, a finite mini-batch
$\mathcal{B}=\{(x^{(i)},y^{*(i)})\}_{i=1}^{N}$ is sampled from the data distribution $\mathcal{T}$.
For each example, the agentic workflow is executed in the to cache node-level inputs and outputs following Algorithm \ref{alg:graph_forward}. 
Then the local targets for sub-modules of the agentic workflow are inferred following Algorithm \ref{alg:gtbp_backward}. Finally, the inputs, cached outputs, and inferred targets are collected across the mini-batch and used to update the tunable prompts following Algorithm~\ref{alg:prompt_update}.

GTBP further controls prompt updates through two stage-wise scheduling
parameters. Let \(m^{0}\) denote the initial edit budget, i.e., the initial
number of prompt claims allowed to be updated. Let \(m^{+}\) denote the
claim-addition interval, meaning that a new prompt claim is added after every
\(m^{+}\) iterations. Let \(m^{-}\) denote the edit-budget decay interval: after every \(m^{-}\)
iterations, the number of editable prompt claims is reduced by one, but never
below one.
\begin{algorithm}[t]
\caption{Graph-based Target Propagation}
\label{alg:gtbp_minibatch}
\small
\begin{algorithmic}[1]
\Require Task distribution $\mathcal{T}$, iterations $T$, batch size $N$,
workflow $\mathcal{G}$, tunable nodes $\mathcal{V}$,
initial prompts $\Pi_{\Phi}^{(0)}$, initial edit budget $n_0$,
claim-addition interval $m^{+}$, edit-budget decay interval $m^{-}$
\Ensure Updated prompts $\Pi_{\Phi}^{(T)}$

\For{$t=1,\ldots,T$}
    \State Sample mini-batch $\mathcal{B}_t \sim \mathcal{T}$
    \For{$(x^{(i)},y^{*(i)})\in\mathcal{B}_t$}
        \State Run forward pass on $\mathcal{G}$ by Algorithm~\ref{alg:graph_forward}
        \State Cache the corresponding $\{(I_v^{(i)},O_v^{(i)})\}_{v\in\mathcal{V}}$
        \State Set $\hat{O}_{v_{\mathrm{out}}}^{(i)} \gets y^{*(i)}$
        \State Infer local targets by Algorithm~\ref{alg:gtbp_backward}
    \EndFor

    \State Set edit budget $n_t$ according to $m^{-}$ and claim-addition marker $a_t \gets 1$ if $t$ is divisible by $m^{+}$, and $a_t \gets 0$ otherwise

    \For{$v\in\mathcal{V}$}
        \State Update $\pi_v^{(t-1)}$ to $\pi_v^{(t)}$ by
        Algorithm~\ref{alg:prompt_update} with edit budget $n_t$ and
        claim-addition marker $a_t$
    \EndFor
\EndFor

\State \Return $\Pi_{\Phi}^{(T)}$
\end{algorithmic}
\end{algorithm}
\section{Theoretical Analysis}
Theoretical analysis of GTBP is challenging because practical LLM-based
agent modules are deep, sequence-level, and text-based. Therefore, we first
introduce a simplified theoretical setting. Under this setting, we analyze two
properties. First, the change of the overall objective in
Eq.~\ref{eq:intro_prompt_optimization} converges to zero as the number of
iterations increases. Second, the prompt-update pipeline induces an
optimization effect similar to gradient descent on trainable weights.
Together, these results explain the stability and descent behavior of the
proposed method. For readability, proofs are deferred to Appendix \ref{sec:appendix}.

\subsection{Simplified Theoretical Setting}
\label{sec:simplified_theory}
\noindent\textbf{Single-token setting.}
We restrict the input-reference pair \((x,y^*)\) in
\(\mathcal{T}\) to a single-token setting. This avoids sequence-level notation
and makes the following loss and stability analysis easier to state.

\noindent\textbf{Prompt-input concatenation.}
For each node \(v\) in the agentic workflow \(\mathcal{G}\), we assume that the node input \(I_v\) and prompt \(\pi_v\) are directly concatenated before being processed by the LLM module, without applying an additional prompt
template.

\noindent\textbf{Single-block module approximation.}
We view each LLM module as a single transformer block:
\[
M_v(I_v;\pi_v)
=
H_{W_v}\!\left(A_v(I_v,\pi_v)\right),
\]
where \(A_v\) maps the concatenated input and prompt to a contextual hidden
representation, and \(H_{W_v}\) denotes the linear transformation with weight \(W_v\).
\subsection{Stability of Prompt Updates}
\label{sec:implicit-weight-bound}

The convergence analysis builds on the implicit weight theorem in
\citet{dherin2025learning}, which shows that changing the input text of a
transformer block can be equivalently represented as a weight update without input text changes.
\begin{lemma}[Implicit Weight]
\label{lem:implicit_patch}
For a transformer block with frozen parameters, let \(t\) denote the update
iteration of GTBP. For a fixed node input \(I_v\), changing the prompt from
\(\pi_v^{(t-1)}\) to \(\pi_v^{(t)}\) can be interpreted as inducing an
input-dependent effective inference-time weight update
\(\Delta W_{v,I_v}^{(t)}\), i.e.,
\begin{equation*}
M_{v,W_v}\!\left(I_v;\pi_v^{(t)}\right)
=
M_{v,W_v+\Delta W_{v,I_v}^{(t)}}\!\left(I_v;\pi_v^{(t-1)}\right).
\end{equation*}
\end{lemma}
\begin{remark}
To save space, we present a simplified version of
Lemma~\ref{lem:implicit_patch} in the main text, while the full statement is
provided in Appendix~\ref{full_lemma1}. The full version is used in the following proofs of theorems, but
is not needed for the main statement.
\end{remark}
Motivated by Theorem~3.2 of \citet{li2023transformers}, the effect of a
single prompt perturbation decreases inversely with prompt length. Since GTBP
accumulates prompt claims over iterations and eventually updates only one
claim at a time, we abstract this behavior into the following stability assumption to upper bound the implicit weight update.

\begin{assumption}[Prompt-length stability]
\label{ass:length_normalized_context}
For each node \(v\) in the agentic workflow \(\mathcal{G}\), 
the induced change in the contextual representation satisfies
\[
\left\|
A_v(I_v,\pi_v^{(t)})
-
A_v(I_v,\pi_v^{(t-1)})
\right\|_2
=
\mathcal{O}(t^{-1}).
\]
\end{assumption}

In Theorem~\ref{lem:edit_budget_bound}, we show that, under Assumption~\ref{ass:nondegeneracy}, the implicit weight increment $\Delta W_{v,I_v}^{(t)}$ whose explicit form is given in Eq.~\ref{eq:implicit_weight_equivalence} can be upper bounded in terms of the update iteration $t$ in GTBP.

\begin{assumption}
\label{ass:nondegeneracy}
There exists $\gamma_v > 0$ such that
\[
\|A_v(I_v,\pi_v^{(t-1)})\|_2 \ge \gamma_v.
\]
\end{assumption}

\begin{theorem}
\label{lem:edit_budget_bound}
With Assumptions~\ref{ass:length_normalized_context} and \ref{ass:nondegeneracy}, the implicit weight update in Lemma \ref{lem:implicit_patch} can be bounded by 
\begin{equation}
\|\Delta W_{v,I_v}^{(t)}\|_F
=
\mathcal{O}\!\left(t^{-1}\right).
\end{equation}
\end{theorem}

Theorem~\ref{lem:edit_budget_bound} shows that the implicit weight update
induced by a stage-wise prompt update becomes smaller as the iteration number
increases. To connect this update stability with the corresponding local objective in Eq.~\ref{eq:prompt_optimizer}, we show in Theorem~\ref{thm:vanishing_loss_difference} that the change of the objective is also
upper bounded.

\begin{assumption}[Per-sample regularity]
\label{ass:per_sample_regularity}
The per-sample local target loss $\ell\!\left(
M_{v}(I_v;\pi_v^{(t)}), \hat O_v
\right)$ in Eq.~\ref{eq:prompt_optimization_local_target} is second-order smooth around \(W_v\), and its gradient
with respect to \(W\) is also bounded along the path from \(W_v\) to
\(W_v+\Delta W_{v,I_v}^{(t)}\).
\end{assumption}
\begin{theorem}
\label{thm:vanishing_loss_difference}
Under Assumptions~\ref{ass:length_normalized_context} to~\ref{ass:per_sample_regularity}, 
if GTBP updates the prompt from $\pi_v^{(t-1)}$ to
$\pi_v^{(t)}$, then the objective change can be upper bounded by 
\begin{equation*}
| F_v(\pi^{(t)}_v) - F_v(\pi^{(t-1)}_v) | = \mathcal{O}(t^{-1}),
\end{equation*}
where $F_v(\cdot)$ is the local target objective in Eq. \ref{eq:prompt_optimization_local_target}:
\begin{equation*}
F_v(\pi_v)=\mathbb{E}_{(x,y^*)\sim\mathcal{T}} \;
\ell\!\left(
M_v(I_v;\pi_v), \hat{O}_v
\right).
\end{equation*}
\end{theorem}

\subsection{Gradient-Aligned Implicit Updates}

Theorem~\ref{thm:vanishing_loss_difference} shows that the local target
optimization updates in GTBP become stable as training proceeds, but stability
alone does not explain why and how the objective can decrease. Therefore, this
subsection studies whether the LLM optimizer can decrease the objective in Eq.~\ref{eq:prompt_optimization} sufficiently well.

GTBP uses LLMs for both backward target inference and prompt updates. Since we
focus here on the effect of prompt updates on the objective, we simplify the
analysis by assuming that the backward operator recovers local targets for
upstream modules, as stated in
Assumption~\ref{ass:backward_optimization_capability}.

\begin{assumption}
\label{ass:backward_optimization_capability}
For each edge $(u,v)\in\mathcal{E}$, the local target inference problem in
Eq.~\ref{eq:target_message_optimization} can be approximately solved by the
LLM-guided backward operator, i.e., with LLM-inferred target output $\hat{O}_v$ and $\hat{O}_u$, we have $\hat{O}_v=M_v\!\big( I_v^{u \leftarrow \hat{O}_u};\pi_v \big)$.
\end{assumption}


Assumption~\ref{ass:backward_optimization_capability} simplifies the backward step by assuming that the LLM can infer local targets consistent with the downstream target. The remaining question is whether the prompt-update step can
solve the local target optimization problem well enough to reduce the overall objective. To prove this, we define the LLM's optimization capability.

Let $\epsilon_{\mathrm{opt},v}$ denote the token-level edit distance between $\pi_v^{\mathrm{LLM}}$ (cf. Eq.~\ref{eq:prompt_optimization_local_target}) and $\pi_v^{\mathrm{opt}}$ (cf. Eq.~\ref{eq:prompt_optimizer}).
A smaller $\epsilon_{\mathrm{opt},v}$ means that the LLM is more capable of generating the prompt that optimizes the local target objective. Under this
condition, we can connect the local target optimization with a descent direction of the overall objective in Theorem \ref{thm:negative_inner_product}.

\begin{theorem}
\label{thm:negative_inner_product}
Under Assumptions~\ref{ass:length_normalized_context}-\ref{ass:backward_optimization_capability}, and for an LLM with sufficiently small $\epsilon_{\mathrm{opt},v}$, the overall loss gap $\big| \mathbb{E}_{(x,y^*)\sim\mathcal{T}}[\ell( \Phi(x;\Pi_{\Phi}^{(t)}), y^* )]-\min_{\Pi_{\Phi}}{\mathbb{E}_{(x,y^*)\sim\mathcal{T}}\left[\ell\!\left(\Phi(x;\Pi_{\Phi}), y^*\right)\right]} \big|$ is $\mathcal{O}(t^{-1})$.  
\end{theorem}


\section{Experiments}
\label{sec:experiments}
\subsection{Datasets and Evaluation Metrics}
We evaluate GTBP on three real-world datasets: SubPOP, HotpotQA, and LiveBench-Math.

\noindent\textbf{SubPOP} \cite{subpop-suh-etal-2025-language}.
SubPOP evaluates whether language models can predict survey response
distributions across population subgroups. Each instance consists of a survey
question, a target subgroup, and the corresponding ground-truth response
distribution. The dataset contains 3,229 training and 133 test questions, each
paired with 22 subgroups. We use Wasserstein distance (WD) to compare predicted
and ground-truth distributions.

\noindent\textbf{HotpotQA} \cite{yang-etal-2018-hotpotqa}. 
HotpotQA evaluates multi-hop question answering, where models must combine
evidence from multiple supporting documents to answer a question. We use the \emph{distractor} setting, in which each instance consists of a question, 2 gold paragraphs containing supporting facts, 8 distractor paragraphs, and a ground-truth answer. We randomly sample 700 questions from the training split for prompt adaptation and 1{,}000 questions from the validation split for evaluation. Following prior work, we report answer-level F1.

\noindent\textbf{LiveBench-Math}  \cite{white2025livebenchchallengingcontaminationlimitedllm}. LiveBench is a popular benchmark with objectively verifiable answers and deterministic graders. Its math subset, LiveBench-Math, contains high-school math-competition problems, symbolic-math problems, and olympiad proof-completion tasks. From the July 30, 2025 snapshot, we shuffle and split the questions into train, validation, and test sets. During evaluation, each held-out question is scored with LiveBench grader, and we report accuracy.

\subsection{Baselines and Hyperparameter Settings}
\label{sec:baselines}

We compare GTBP with three baseline methods under the same backbone model, GPT-4.1-mini:
\noindent\textbf{GTBP (proposed).}
Prompts in GTBP are automatically initialized from data rather than manual design. For each intermediate module $M_k$, we independently sample a random mini-batch and use an LLM to summarize its task patterns, reasoning skills, and output requirements from a distinct perspective. The resulting summary is used to initialize the prompt $\pi_k$ of module $M_k$. We then summarize the first-layer prompts $\{\pi_k\}_{k=1}^{K}$ to generate the output-module prompt $\pi_{\mathrm{out}}$, instructing the output module to integrate the intermediate responses.

We instantiate GTBP with the described workflow and tune the number of intermediate modules over $K\in\{1,2,4\}$. 
The mini-batch size and the number of iterations are both set as 15. For prompt update, the initial number of claims in each
prompt is set to 10, and both $m^+$ and $m^-$ are set to 4.

\noindent\textbf{GEPA.} We run GEPA using the DSPy implementation \cite{khattab2024dspy} with Pareto-based candidate selection and update skipping on perfect-scoring examples. For HotpotQA, we directly configure the optimizer hyperparameters with a reflection mini-batch size of 3, a metric-call budget of 1000.
For SubPOP, we use the \texttt{light} auto configuration with a reflection mini-batch size of 15, round-robin component selection, a merge phase with up to 5 merge invocations.

\noindent\textbf{Zero- and Few-shot.}
Zero-shot directly queries LLM with the task input. Few-shot adds randomly sampled training examples before the test query as
in-context demonstrations. For HotpotQA and LiveBench-Math, where inputs are
longer, we tune the number of examples over $\{1,3,5,10\}$. For SubPOP, we maintain consistency as \cite{subpop-suh-etal-2025-language} and keep the number of examples as 5.

\begin{table}[t]
\centering
\small
\setlength{\tabcolsep}{5pt}
\begin{tabular}{lccc}
\toprule
\textbf{Method} 
& \textbf{HotpotQA} 
& \textbf{LiveBench-Math} 
& \textbf{SubPOP} \\
\cmidrule(lr){2-2}
\cmidrule(lr){3-3}
\cmidrule(lr){4-4}
& \textbf{F1 $\uparrow$} 
& \textbf{Acc. $\uparrow$} 
& \textbf{WD $\downarrow$} \\
\midrule
Zero-shot  & 0.7535 & 0.6111 & 0.1261 \\
Few-shot   & 0.7659 & 0.5556 & 0.1212 \\
GEPA       & 0.7697     & 0.6413 & 0.1160 \\
GTBP       &\textbf{ 0.80}   & \textbf{0.6690} &  \textbf{0.1072} \\
\bottomrule
\end{tabular}
\caption{
Performance comparison across HotpotQA, LiveBench-Math, and SubPOP, each reported using its standard evaluation metric.
}
\label{tab:main_results}
\end{table}

\begin{table}[t]
\centering
\small
\setlength{\tabcolsep}{4pt}
\begin{tabular}{lcccccc}
\toprule
\textbf{Dataset}
& \multicolumn{3}{c}{\textbf{GEPA}}
& \multicolumn{3}{c}{\textbf{GTBP}} \\
\cmidrule(lr){2-4}
\cmidrule(lr){5-7}
& \textbf{\#Calls} & \textbf{In} & \textbf{Out}
& \textbf{\#Calls} & \textbf{In} & \textbf{Out} \\
\midrule
HotpotQA      & 1{,}080 & 2{,}374 & 36 & 2{,}122 & 9{,}826 & 335 \\
LiveBench & 1{,}807 & 2{,}071 & 1{,}594 & 747 & 3{,}580 & 932 \\
SubPOP      & 12,700 & 1,112 & 219 & 1,639 & 1,730  & 456\\
\bottomrule
\end{tabular}
\caption{
Training efficiency: total LLM calls per run and average input/output tokens per call.
}
\label{tab:training_efficiency}
\end{table}
\begin{table}[t]
\centering
\small
\setlength{\tabcolsep}{4pt}
\begin{tabular}{lccc}
\toprule
\textbf{Method} 
& \textbf{HotpotQA} 
& \textbf{LiveBench-Math} 
& \textbf{SubPOP} \\
\cmidrule(lr){2-2}
\cmidrule(lr){3-3}
\cmidrule(lr){4-4}
& \textbf{In / Out} 
& \textbf{In / Out} 
& \textbf{In / Out} \\
\midrule
Zero-shot & 1{,}629 / 11 & 364 / 1{,}741 & 169 / 26 \\
Few-shot  & 15{,}296 / 11 & 1{,}015 / 1{,}442 & 641 / 27 \\
GEPA      & 2{,}019 / 16 & 1{,}683 / 1{,}895 & 1,055 / 210 \\
GTBP      & 6{,}358 / 278 & 4{,}146 / 709 & 2,176 / 328 \\
\bottomrule
\end{tabular}
\caption{
Inference efficiency: average per-sample input/output token usage across all methods.}
\label{tab:inference_efficiency}
\end{table}
\subsection{Effectiveness and efficiency evaluation}
Table~\ref{tab:main_results} shows that GTBP achieves the best overall performance across all three benchmarks. Relative to the strongest baseline, GTBP improves HotpotQA F1 from 0.769 to 0.800 (+4.0\%), LiveBench-Math accuracy from 0.641 to 0.669 (+4.4\%), and reduces SubPOP WD from 0.116 to 0.107 (-7.8\%). 
These gains are consistent across both reasoning tasks (HotpotQA, LiveBench-Math) and distribution prediction (SubPOP), suggesting that explicitly structured, module-level credit signals lead to more effective prompt updates than global reflection-based approaches. 

Table~\ref{tab:training_efficiency} compares GEPA and GTBP on LLM calls per training run and average input/output tokens per call. On HotpotQA, GTBP issues more calls and produces more output tokens, reflecting the structured credit signals generated during backward target inference. On LiveBench-Math and SubPOP, GTBP requires far fewer calls than GEPA, demonstrating that module-level target propagation can be significantly more call-efficient than global candidate generation on tasks with higher per-call overhead. Overall, GTBP achieves comparable or lower training costs than GEPA while consistently delivering stronger task performance. 

Table~\ref{tab:inference_efficiency} reports per-sample inference token usage. GTBP incurs higher per-sample token usage than zero- and few-shot baselines, but remains in the same order of magnitude as GEPA. The higher output token count of GTBP indicates module-level reasoning. 
Since GTBP operates entirely in prompt space with frozen model weights, it is compatible with proprietary LLMs without requiring gradient access. The resulting prompts are human-readable, as illustrated by the SubPOP examples in App.~\ref{app:subpop_results}.

\subsection{Convergence Analysis}
\label{sec:convergence_analysis}
\begin{figure}
    \centering
    \includegraphics[width=1\linewidth]{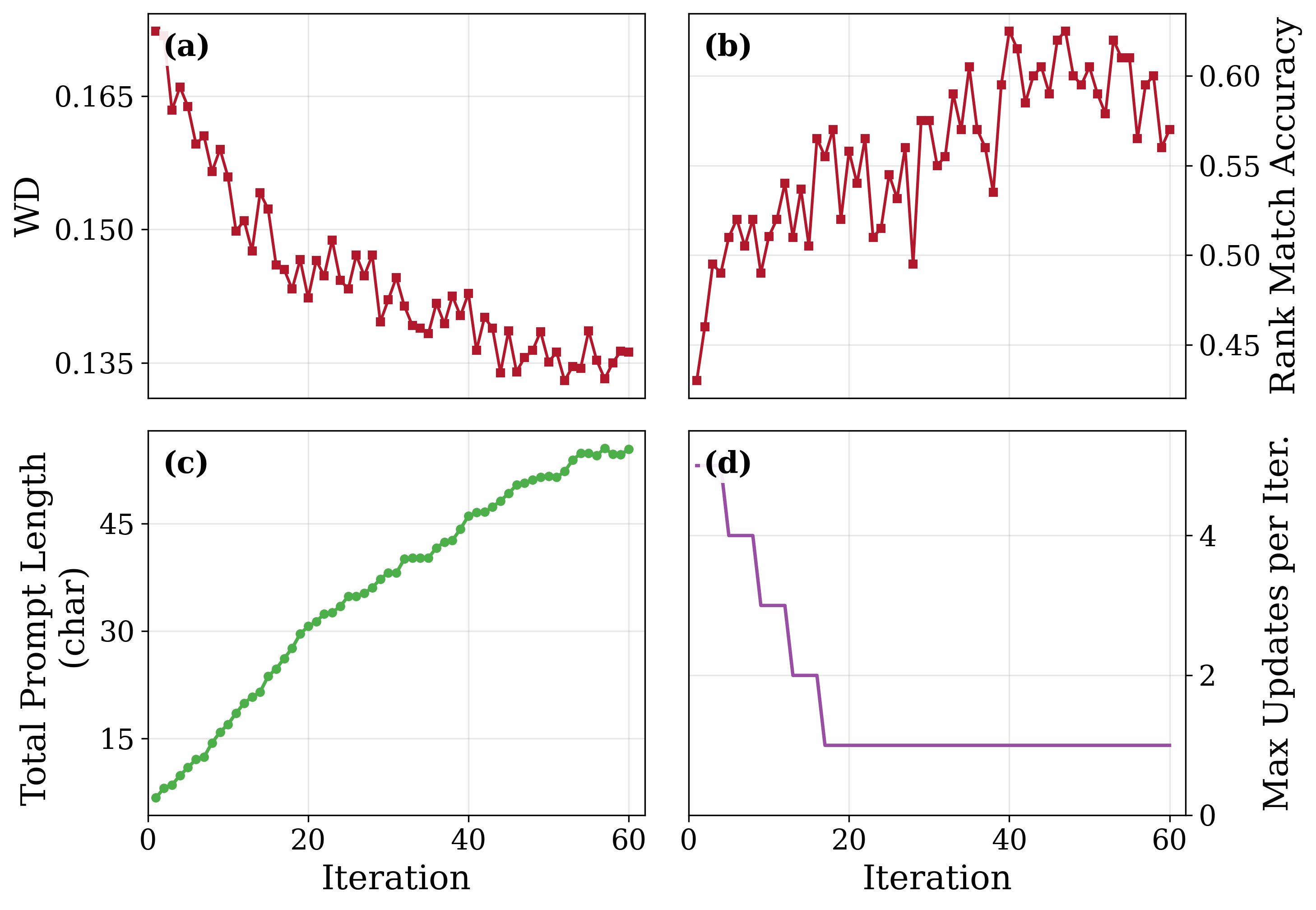}
    \caption{GTBP training convergence over 60 iterations. (a) Training Wasserstein distance. (b) Training rank match accuracy. (c) Total prompt length across network nodes. (d) Maximum update per iteration.}
    \label{fig:convergence_analysis}
\end{figure}

We validate empirically Theorem~\ref{thm:vanishing_loss_difference}. As training proceeds and prompt edits become smaller, the induced workflow change gradually stabilizes.

We analyze SubPOP with 200 samples from a single subgroup. We use SubPOP
because its target distributions are explicitly available, enabling rule-based
evaluation without an LLM judge. During GTBP training, we track WD,
rank-match accuracy, prompt length, and the maximum number of prompts updated
per iteration. Rank-match accuracy measures whether the ordering of predicted
options matches the ordering induced by the target distribution, providing a
discrete measure of distributional alignment in addition to WD.

Figure~\ref{fig:convergence_analysis} reveals that WD decreases and rank-match accuracy rises rapidly in the early iterations, after which both metrics gradually stabilize.
Meanwhile, prompt length grows monotonically while the update ratio decreases, indicating that GTBP accumulates structured knowledge through progressively smaller edits, promoting stable convergence.

\section{Conclusion and Future Work}

We propose GTBP, a graph-based target back-propagation framework for context
adaptation in multi-LLM agentic systems. GTBP assigns credit by inferring local
target outputs for intermediate modules and updates prompts using
module-level discrepancies, while keeping LLM weights frozen. Experiments on
reasoning and opinion prediction tasks show that GTBP improves over prompting
and prompt-optimization baselines, suggesting that target-based credit
assignment is useful for black-box LLM workflows. Future work will extend GTBP
to broader agentic abilities, such as memory retrieval, tool calling, and
environment interaction, as well as more complex workflow structures, including
deeper graphs and dynamic routing.

\section*{Limitations}

While GTBP shows consistent improvements across HotpotQA, LiveBench-Math, and SubPOP, several limitations remain and suggest directions for future work.

Our theoretical analysis relies on several assumptions to make the optimization
dynamics tractable, including bounded prompt edits and sufficiently accurate
LLM-based optimizer. In particular, the convergence analysis
assumes the existence of a sufficiently capable LLM optimizer that can produce
prompt updates closely aligned with the inferred local targets. While these
assumptions enable formal analysis of the proposed target-propagation
mechanism, extending the results to weaker assumptions and more general
workflow settings remains an important direction for future work.

Our proposed work focuses on the context-adaptation setting in which all LLM weights are frozen. In some settings, GTBP underperforms fine-tuned models, such as the fine-tuned Llama-3-70B results reported on SubPOP (Table~\ref{tab:demographic_results}). Fine-tuning typically relies on open-weight models and can achieve strong performance on individual benchmarks after training on task-specific datasets, potentially at the cost of broader task generalization. GTBP instead aims to provide a more general adaptation
framework while keeping model weights frozen.

The agentic workflow we evaluate is a simplified instance of what is possible in practice. Our experiments instantiate GTBP on a one-hidden-layer workflow with up to four intermediate modules and an output node (Section~\ref{sec:baselines}). The proposed structure avoids the multi-target ambiguity that arises when a node branches to multiple successors, and it does not capture deeper graphs, tool-calling, or dynamic routing that are common in real-world agentic systems. Whether the target back-propagation dynamics scale gracefully to such structures remains an open empirical question.

The prompt-update mechanism itself also leaves room for improvement. GTBP
iteratively adds claims to the prompt, causing the total prompt length to grow
over training, as shown in Figure~\ref{fig:convergence_analysis}(c). In this
work, we did not explicitly optimize for prompt brevity or per-call token cost.
Future work could explore prompt-budgeting schemes, such as merging or pruning
redundant claims, jointly optimizing task loss and prompt length, or introducing
memory-retrieval mechanisms that selectively retrieve relevant claims at
inference time, thereby reducing token cost without sacrificing accuracy.

\section*{Ethical Considerations}
This work does not raise significant ethical concerns. All datasets used in our experiments, HotpotQA, LiveBench-Math, and SubPOP, are publicly available and contain no personally identifiable information. The SubPOP dataset consists of anonymized survey responses aggregated at the demographic subgroup level. No new data collection or human annotation was conducted as part of this work.

\section*{Acknowledgment}
We acknowledge the use of AI writing assistants to enhance writing clarity and AI coding assistants to support the implementation of experiments.

\bibliography{acl_latex}

\appendix

\section{Appendix of Theoretical Analysis}
\label{sec:appendix}
\subsection{Full version of Lemma \ref{lem:implicit_patch}}
\label{full_lemma1}

\begin{lemmafull}
For a transformer block with frozen parameters, let $t$ denote the update
iteration of GTBP. Changing the prompt from $\pi_v^{(t-1)}$ to $\pi_v^{(t)}$
can be interpreted as inducing an effective inference-time weight update
$\Delta W_v^{(t)}$, i.e., 
\begin{equation}
\label{eq:implicit_weight_equivalence}
T_{W_v}\!\left(\pi_v^{(t)}, x_v\right)
=
T_{W_v+\Delta W_v^{(t)}}\!\left(\pi_v^{(t-1)}, x_v\right),
\end{equation}
where 
\begin{equation}
\label{delta_def}
\Delta W_v^{(t)}=\frac{
\left(
W_v\,\delta A_v\!\left(\pi_v^{(t)}, \pi_v^{(t-1)}\right)
\right)
A_v(\pi_v^{(t-1)})^\top
}{
\|A_v(\pi_v^{(t-1)})\|_2^2
}
\end{equation}
with 
\begin{equation*}
\delta A_v\!\left(\pi_v^{(t)}, \pi_v^{(t-1)}\right)
:=
A_v\!\left(\pi_v^{(t)}\right) - A_v\!\left(\pi_v^{(t-1)}\right).
\end{equation*}
\end{lemmafull}

\subsection{Proof of Theorem \ref{lem:edit_budget_bound}}
\label{app:edit_budget_bound}
\begin{proof}
Taking the Frobenius norm on both sides of Eq.~\ref{delta_def} and using
\(\|uv^\top\|_F=\|u\|_2\|v\|_2\), we obtain
\begin{align*}
    &\|\Delta W_v^{(t)}\|_F\\
    &=
    \frac{
    \left\|W_v\bigl(
    A_v(\pi_v^{(t)},x_v)
    -
    A_v(\pi_v^{(t-1)},x_v)
    \bigr)\right\|_2
    }
    {
    \|A_v(\pi_v^{(t-1)},x_v)\|_2
    }\\
    &\le
    \frac{
    \|W_v\|_2
    \|A_v(\pi_v^{(t)},x_v)-A_v(\pi_v^{(t-1)},x_v)\|_2
    }{
    \|A_v(\pi_v^{(t-1)},x_v)\|_2
    }.
\end{align*}
By Assumption~\ref{ass:length_normalized_context},
\[
\|A_v(\pi_v^{(t)},x_v)-A_v(\pi_v^{(t-1)},x_v)\|_2
=
\mathcal{O}(t^{-1}).
\]
By Assumption~\ref{ass:nondegeneracy},
\[
\|A_v(\pi_v^{(t-1)},x_v)\|_2 \ge \gamma_v.
\]
Therefore,
\[
\|\Delta W_v^{(t)}\|_F
\le
\frac{\|W_v\|_2}{\gamma_v} \mathcal{O}(t^{-1}).
\]
Since \(W_v\) is frozen, \(\|W_v\|_2\) is a constant. Hence,
\[
\|\Delta W_v^{(t)}\|_F
=
\mathcal{O}(t^{-1}).
\]
\end{proof}
\subsection{Proof of Theorem~\ref{thm:vanishing_loss_difference}}
\label{vanishing_loss_appendix}
\begin{proof}
By Lemma~\ref{lem:implicit_patch}, the prompt update from
\(\pi_v^{(t-1)}\) to \(\pi_v^{(t)}\) can be represented as an
input-dependent effective weight update \(\Delta W_{v,I_v}^{(t)}\). Thus,
for the local target loss, we can write
\begin{multline*}
    \ell\!\left(
    M_v(I_v;\pi_v^{(t)}), \hat O_v
    \right) \\
    =
    \ell\!\left(
    M_{v,W_v+\Delta W_{v,I_v}^{(t)}}(I_v;\pi_v^{(t-1)}),
    \hat O_v
    \right).
\end{multline*}
Define the per-sample local target loss as
\[
f_{v,t-1}(W_v)
:=
\ell\!\left(
M_{v}(I_v;\pi_v^{(t-1)}), \hat O_v
\right).
\]
Then, using $\mathbb{E}$ as shorthand for $\mathbb{E}_{(x,y^*)\sim\mathcal{T}}$, we have  
\begin{multline}
    \label{eq:diff_Fv}
    F_v(\pi_v^{(t)})-F_v(\pi_v^{(t-1)}) \\
    =
    \mathbb{E}
    \left[
    f_{v,t-1}(W_v+\Delta W_{v,I_v}^{(t)})
    -
    f_{v,t-1}(W_v)
    \right].
\end{multline}

By Assumption~\ref{ass:per_sample_regularity}, \(f_{v,t-1}\) is second-order
smooth along the segment from \(W_v\) to
\(W_v+\Delta W_{v,I_v}^{(t)}\). Therefore, 
\begin{multline}
    \label{eq:apply_smoothness}
    \left|
    f_{v,t-1}(W_v+\Delta W_{v,I_v}^{(t)})
    -
    f_{v,t-1}(W_v)
    \right| \\
    \le
    \left|
    \left\langle
    \nabla f_{v,t-1}(W_v),
    \Delta W_{v,I_v}^{(t)}
    \right\rangle
    \right|
    +
    \frac{L_v}{2}
    \left\|\Delta W_{v,I_v}^{(t)}\right\|_F^2 .
\end{multline}
By the bounded-gradient condition in
Assumption~\ref{ass:per_sample_regularity}, we assume that there exists
\(G_v>0\) such that
\[
\left\|
\nabla f_{v,t-1}(W_v)
\right\|_F
\le
G_v .
\]
Then, by Cauchy--Schwarz, the first term on the right-hand side of
Eq.~\ref{eq:apply_smoothness} satisfies
\begin{align*}
    &\left|
    \left\langle
    \nabla f_{v,t-1}(W_v),
    \Delta W_{v,I_v}^{(t)}
    \right\rangle
    \right|\\
    &\le
    \|\nabla f_{v,t-1}(W_v)\|_F
    \left\|\Delta W_{v,I_v}^{(t)}\right\|_F \\
    &\le
    G_v
    \left\|\Delta W_{v,I_v}^{(t)}\right\|_F .
\end{align*}
Hence, we can rewrite Eq.~\ref{eq:apply_smoothness} as
\begin{multline}
    \label{eq:apply_smoothness_2}
    \left|
    f_{v,t-1}(W_v+\Delta W_{v,I_v}^{(t)})
    -
    f_{v,t-1}(W_v)
    \right| \\
    \le
    G_v
    \left\|\Delta W_{v,I_v}^{(t)}\right\|_F
    +
    \frac{L_v}{2}
    \left\|\Delta W_{v,I_v}^{(t)}\right\|_F^2 .
\end{multline}

By Theorem~\ref{lem:edit_budget_bound}, we have that
\[
\left\|\Delta W_{v,I_v}^{(t)}\right\|_F
=
\mathcal{O}(t^{-1}),
\]
which leads to 
\begin{equation}
    \label{eq:first_moment_diffW}
    \mathbb{E}
    \left[
    \left\|\Delta W_{v,I_v}^{(t)}\right\|_F
    \right]
    =
    \mathcal{O}(t^{-1}),
\end{equation}
and
\begin{equation}
    \label{eq:second_moment_diffW}
    \mathbb{E}
    \left[
    \left\|\Delta W_{v,I_v}^{(t)}\right\|_F^2
    \right]
    =
    \mathcal{O}(t^{-2}).
\end{equation}
Therefore, by combining Eqs.~\ref{eq:diff_Fv}, \ref{eq:apply_smoothness_2}, \ref{eq:first_moment_diffW}, and \ref{eq:second_moment_diffW}, we obtain
\begin{align*}
    &\left|
    F_v(\pi_v^{(t)})-F_v(\pi_v^{(t-1)})
    \right| \\
    &\le
    \mathbb{E}
    \left[
    \left|
    f_{v,t-1}(W_v+\Delta W_{v,I_v}^{(t)})
    -
    f_{v,t-1}(W_v)
    \right|
    \right] \\
    &=
    \mathcal{O}(t^{-1})+\mathcal{O}(t^{-2})\\
    &=
    \mathcal{O}(t^{-1}).
\end{align*}

\end{proof}
\subsection{Proof of Theorem~\ref{thm:negative_inner_product}}
\label{app:negative_inner_product}

\begin{proof}
Let
\[
F(\Pi_\Phi)
=
\mathbb{E}_{(x,y^*)\sim\mathcal{T}}
\left[
\ell\!\left(\Phi(x;\Pi_\Phi),y^*\right)
\right]
\]
denote the overall objective in Eq.~\ref{eq:prompt_optimization}. For the
one-hidden-layer workflow, the intermediate responses at iteration \(t\) are
\[
O_k^{(t)} = M_k(x;\pi_k^{(t)}), \qquad k=1,\dots,K,
\]
and the final output is
\[
y^{(t)}
=
M_{\mathrm{out}}
\left(
x,O_1^{(t)},\dots,O_K^{(t)};\pi_{\mathrm{out}}^{(t)}
\right).
\]
Consider the prompt-update step for the first-layer modules. Sufficiently small
\(\epsilon_{\mathrm{opt},k}\) means that the LLM optimizer is powerful enough
to generate prompts whose outputs are close to the inferred local targets. Thus, for each first-layer node \(k\), the updated response becomes
sufficiently close to its inferred local target \(\hat O_k\).
Since \(K\) is finite, the aggregated input received by the output node 
\[
I_{\mathrm{out}}^{(t)}
=
(x,O_1^{(t)},\dots,O_K^{(t)})
\]
also
close to the target-consistent input 
\[
\hat I_{\mathrm{out}}
=
(x,\hat O_1,\dots,\hat O_K).
\]
By Assumption~\ref{ass:length_normalized_context} applied to the output node,
the contextual representation change induced by the updated output-node input
and prompt is \(\mathcal{O}(t^{-1})\). Therefore, by
Theorem~\ref{thm:vanishing_loss_difference}, the corresponding change in the
output-node local target objective is also \(\mathcal{O}(t^{-1})\):
\begin{equation}
    \label{eq:diff_Fout_bound}
    \left|
    F_{\mathrm{out}}(I_{\mathrm{out}}^{(t)},\pi_{\mathrm{out}}^{(t)})
    -
    F_{\mathrm{out}}(\hat I_{\mathrm{out}},\pi_{\mathrm{out}}^{(t-1)})
    \right|
    =
    \mathcal{O}(t^{-1}).
\end{equation}
With Assumption \ref{ass:backward_optimization_capability}, with $\hat{I}_{\mathrm{out}}$ we have 
\begin{equation*}
    y^{*}=M_{\mathrm{out}}\left(x,\hat{O}_1,\dots,\hat{O}_K;\pi_{\mathrm{out}}^{(t-1)}\right),
\end{equation*}
which results in 
\begin{equation}
    \label{eq:Fout_t-1}
    F_{\mathrm{out}}(\hat I_{\mathrm{out}},\pi_{\mathrm{out}}^{(t-1)})=F(\Pi_\Phi^*).
\end{equation}
Since 
\begin{equation*}
    F_{\mathrm{out}}( I_{\mathrm{out}}^{(t)},\pi_{\mathrm{out}}^{(t)})=F(\Pi_\Phi^{(t)}),
\end{equation*}
applying Eq.~\ref{eq:Fout_t-1}, Eq.~\ref{eq:diff_Fout_bound} becomes
\begin{equation*}
    |F(\Pi_\Phi^{(t)})-F(\Pi_\Phi^*)|=\mathcal{O}(t^{-1}).
\end{equation*}
\end{proof}
\section{Appendix of Experiment}
\label{app:subpop_results}
This section expands on the methodology and results introduced in Section~\ref{sec:experiments}, providing group-level breakdowns across demographic categories for the SubPOP dataset~\cite{subpop-suh-etal-2025-language}. 

Table~\ref{tab:demographic_results} compares our proposed algorithm against the self-reported performance metrics in~\cite{subpop-suh-etal-2025-language} and several baseline approaches.

\noindent\textbf{Few-shot prompting}: We consider a conditioning prompt that includes group information along with the group's response distribution over training questions, followed by the test question as the instruction prompt. In our experiments, we use 5 examples as the few-shot baseline.

\noindent\textbf{Fine-tuned model}: We report benchmark metrics from \cite{subpop-suh-etal-2025-language}, where the authors fine-tuned several LLMs of varying parameter counts using the training data.

\noindent\textbf{GEPA}: We leverage GEPA~\cite{agrawal2025gepa} as a prompt training optimizer. 
Example prompts trained using GEPA is provided in Appendix.

\noindent\textbf{GTBP}: Example prompts trained using GTBP are provided in Appendix.  

Our proposed approach (GTBP) achieves a lower average Wasserstein Distance (WD) than GEPA across demographic groups, demonstrating consistent improvements in bias mitigation without task-specific fine-tuning. Furthermore, when compared against fine-tuned models of similar parameter counts, specifically Llama-2-13B and Mistral-7B-v0.1, our approach still attains competitive or superior performance, highlighting its effectiveness as a model weight training alternative.

However, our method does not outperform fine-tuned Llama-3-70B on average WD. We attribute this gap to two complementary factors: (1) the inherent representational capacity of larger models, which may enable more nuanced demographic inference by default, and (2) the benefit of supervised fine-tuning, which directly optimizes the model toward the target distribution. These results suggest that scaling model size or incorporating fine-tuning remain important considerations for future work.

\begin{table*}[t]
\centering
\small
\setlength{\tabcolsep}{4pt}
\begin{adjustbox}{width=\textwidth}
\begin{tabular}{cccccccccccccc}
\multirow{2}{*}{Attribute}          & \multirow{2}{*}{Group} & \multirow{2}{*}{Human Baseline} & \multicolumn{2}{c}{Llama-2-7B} & \multicolumn{2}{c}{Llama-2-13B} & \multicolumn{2}{c}{Mistral-7B-v0.1} & \multicolumn{3}{c}{GPT-4.1-mini}           & \multicolumn{2}{c}{Llama-3-70B} \\
                                    &                        &                                 & Base        & Fine-tuned       & Base         & Fine-tuned       & Base         & Fine-tuned           & Base   & GEPA            & \textbf{GTBP}   & Base           & Fine-tuned     \\ \hline
\multirow{2}{*}{Region}             & Northeast              & 0.027                           & 0.196       & 0.113            & 0.193        & 0.103            & 0.185        & 0.108                & 0.0991 & 0.0982          & \textbf{0.0952} & 0.156          & 0.078          \\
                                    & South                  & 0.018                           & 0.183       & 0.108            & 0.185        & 0.103            & 0.176        & 0.103                & 0.1168 & \textbf{0.0994} & 0.101           & 0.138          & 0.08           \\
\multirow{2}{*}{Education}          & College grad*          & 0.019                           & 0.206       & 0.105            & 0.175        & 0.101            & 0.167        & 0.099                & 0.0954 & 0.0864          & \textbf{0.084}  & 0.137          & 0.077          \\
                                    & Less than high school  & 0.036                           & 0.191       & 0.129            & 0.182        & 0.117            & 0.172        & 0.121                & 0.1326 & 0.119           & \textbf{0.112}  & 0.18           & 0.108          \\
\multirow{2}{*}{Gender}             & Male                   & 0.017                           & 0.186       & 0.102            & 0.176        & 0.101            & 0.17         & 0.099                & 0.1063 & \textbf{0.0904} & 0.092           & 0.15           & 0.079          \\
                                    & Female                 & 0.016                           & 0.184       & 0.108            & 0.198        & 0.105            & 0.176        & 0.1                  & 0.105  & \textbf{0.084}  & 0.094           & 0.151          & 0.08           \\
\multirow{4}{*}{Race / ethnicity}   & Black                  & 0.029                           & 0.2         & 0.114            & 0.179        & 0.102            & 0.17         & 0.107                & 0.1099 & 0.124           & \textbf{0.089}  & 0.139 &  0.094 \\
                                    & White                  & 0.014                           & 0.19        & 0.105            & 0.187        & 0.103            & 0.181        & 0.102                & 0.1    & \textbf{0.0905} & 0.0922          & 0.153          & 0.083          \\
                                    & Asian                  & 0.049                           & 0.201       & 0.119            & 0.19         & 0.107            & 0.184        & 0.114                & 0.1158 & 0.113           & \textbf{0.107}  & 0.158          & 0.096          \\
                                    & Hispanic               & 0.05                            & 0.204       & 0.133            & 0.199        & 0.122            & 0.182        & 0.134                & 0.1377 & 0.13            & \textbf{0.1198} & 0.172          & 0.115          \\
\multirow{2}{*}{Income}             & \$100,000 or more      & 0.021                           & 0.21        & 0.111            & 0.184        & 0.106            & 0.176        & 0.102                & 0.1145 & \textbf{0.0937} & 0.0985          & 0.179          & 0.082          \\
                                    & Less than \$30,000     & 0.026                           & 0.179       & 0.115            & 0.172        & 0.103            & 0.165        & 0.105                & 0.1083 & 0.102           & \textbf{0.099}  & 0.171          & 0.086          \\
\multirow{2}{*}{Political Party}    & Democrat               & 0.02                            & 0.219       & 0.103            & 0.197        & 0.092            & 0.199        & 0.091                & 0.1011 & 0.105           & \textbf{0.0816} & 0.128          & 0.076          \\
                                    & Republican             & 0.023                           & 0.205       & 0.123            & 0.234        & 0.117            & 0.206        & 0.115                & 0.1321 & 0.144           & \textbf{0.109}  & 0.187          & 0.093          \\
\multirow{3}{*}{Political Ideology} & Liberal                & 0.019                           & 0.224       & 0.102            & 0.191        & 0.09             & 0.188        & 0.096                & 0.1034 & 0.109           & \textbf{0.09}   & 0.134          & 0.076          \\
                                    & Conservative           & 0.022                           & 0.184       & 0.12             & 0.178        & 0.112            & 0.172        & \textbf{0.113}       & 0.1425 & 0.142           & 0.136           & 0.16           & 0.092          \\
                                    & Moderate               & 0.018                           & 0.191       & 0.11             & 0.183        & 0.103            & 0.17         & 0.103                & 0.0989 & \textbf{0.092}  & 0.095           & 0.141          & 0.082          \\
\multirow{5}{*}{Religion}           & Protestant             & 0.019                           & 0.187       & 0.11             & 0.179        & 0.107            & 0.172        & \textbf{0.105}       & 0.1318 & 0.112           & 0.1111          & 0.164          & 0.082          \\
                                    & Jewish                 & 0.066                           & 0.245       & 0.149            & 0.226        & 0.144            & 0.218        & 0.129                & 0.1157 & 0.149           & \textbf{0.113}  & 0.164 & 0.119 \\
                                    & Hindu                  & 0.095                           & 0.264       & 0.18             & 0.253        & 0.169            & 0.252        & 0.186                & 0.2107 & \textbf{0.177}  & 0.18            & 0.223          & 0.166          \\
                                    & Atheist                & 0.021                           & 0.222       & 0.126            & 0.207        & 0.103            & 0.199        & 0.116                & 0.1135 & 0.113           & \textbf{0.095}  & 0.132          & 0.106          \\
                                    & Muslim                 & 0.09                            & 0.253       & 0.175            & 0.24         & 0.181            & 0.238        & 0.173                & 0.1829 & 0.179           & \textbf{0.163}  & 0.203          & 0.158          \\ \hline
\multicolumn{2}{c}{\textbf{Average WD}}                      & 0.033                           & 0.206       & 0.121            & 0.196        & 0.113            & 0.187        & 0.115                & 0.122  & 0.116           & \textbf{0.107}  & 0.160          & 0.096         
\end{tabular}
\end{adjustbox}
\caption{Bias/divergence scores across demographic groups and model variants.}
\label{tab:demographic_results}
\end{table*}

\onecolumn

\begin{lstlisting}[
    caption={Example of GEPA trained prompt for HotpotQA},
    label={lst:gepa hotpot},
    basicstyle=\ttfamily\small,
    breaklines=true,
    frame=single
]
This agent answers multi-hop questions from the HotpotQA benchmark where each question is paired with multiple supporting context paragraphs containing relevant information.

Task Overview:
- The questions focus on two main reasoning patterns:
  1. Bridge Questions: Require connecting intermediate entities or facts across different context pieces to reach the final answer.
  2. Comparison Questions: Require comparing two entities based on explicit, extractable attributes from the context to determine a yes/no answer or select the entity that better meets criteria.

Requirements for Processing:
1. Identify all relevant entities referenced in the question from the context.
2. Extract explicit, directly stated attributes or facts related to these entities that are needed to answer the question.
3. If the question involves bridging information, connect facts across multiple contexts logically and sequentially.
4. For comparison questions, isolate and compare the explicit attributes of the two entities without overgeneralization or inferential leaps.
5. Provide a final answer that is concise and precise, limited to:
   - A short phrase,
   - A specific entity name,
   - Or a direct "yes" or "no" answer depending on the question.

Key Domain-Specific Guidance:
- When identifying shared fields, professions, or attributes between two entities, only use explicitly stated, precise commonalities in the text. Avoid broad, implicit commonalities. For example, prefer "ESA astronaut" rather than the broader "astronauts" if that specificity is provided.
- For questions involving modes or means (e.g., travel methods), respond only with the mode itself, excluding ancillary context.
- When asked about something like starring roles or main works associated with an individual, distinguish starring/main roles from minor appearances or voice acting. Use only explicitly identified starring roles as answers.
- Avoid any explanatory or background details in the final output. The answer must be the succinct final solution to the question without added context.

Answer Formatting:
- Provide the answer in the form requested by the question (name, phrase, yes/no).
- Do not include supporting explanations, justifications, or references.
  
Overall Approach for Multi-hop Reasoning:
- Parse the question to identify entities, relationships, or comparison targets.
- Extract attributes for these entities from the contexts.
- Apply bridge reasoning by linking intermediate information when required.
- For comparisons, explicitly isolate and evaluate attributes based on context-provided facts.
- Produce a clean, precise final answer strictly aligned with the wording and scope of the question.

This precision and structured approach ensure accurate and concise answers on multi-hop and comparison questions that require synthesizing multiple context paragraphs with domain-specific understanding.
\end{lstlisting}

\begin{lstlisting}[
    caption={Example of GTBP trained prompt for HotpotQA - Node $L_0\_N_0$},
    label={lst:gtbp hotpot prompt l0n0},
    basicstyle=\ttfamily\small,
    breaklines=true,
    frame=single
]
- Use the exact adjective form for nationality as expected by the question, such as 'English' instead of 'British' and 'American' instead of 'United States'.  
- Verify the precise chemical compound name when synonyms or colloquial terms are given, ensuring exact identification rather than a related analogue.  
- For questions about multiple species coexisting, prefer broader regions (e.g., 'Africa') over specific countries unless all species are explicitly confirmed there.  
- For blockbuster film questions by year, select a film explicitly listed as a blockbuster in that year and different from the one named in the question.  
- When asked about the meaning or origin of a term, confirm if the expected answer is the literal meaning or the entity itself, and respond accordingly.  
- For geographic questions involving suburbs and cities, provide the official city or local government area name rather than just the city name.  
- Provide full exact dates (day, month, year) when a specific date is requested, not just the year.  
- Focus on the most specific geographic or administrative entity mentioned when multiple levels are referenced.  
- Extract multi-dimensional patterns from question-answer pairs involving named entities, cultural references, and factual data across domains like entertainment, history, geography, and science; explicitly state common genres and clarify genre relationships, especially subgenres.  
- Identify entity types (persons, places, organizations, works of art) and their relationships, including roles (actor, director, author), affiliations (bands, universities, companies), and temporal markers (birthdates, event dates).
\end{lstlisting}

\begin{lstlisting}[
    caption={Example of GTBP trained prompt for HotpotQA - Node $L_0\_N_1$},
    label={lst:gtbp hotpot prompt l0n1},
    basicstyle=\ttfamily\small,
    breaklines=true,
    frame=single
]
- Extract nuanced entity relationships and temporal dynamics from question-answer pairs across entertainment, history, geography, and science.  
- Identify precise entity types-individuals, creative works, organizations, locations-and their interrelations, including roles, affiliations, biographical, or production details; always use full official names exactly as expected, avoiding abbreviations unless explicitly requested.  
- For biographical questions, provide the full, most specific entity name preferred by the target, including middle names, initials, birth and death dates, and explicitly name spouses with full or commonly recognized titles.  
- For creative works, analyze production credits, release years, source materials, and genres to predict concise factual answers naming key contributors or exact titles, including subtitles, punctuation, and release years when relevant.  
- For geographic or institutional queries, match the requested geographic specificity-distinguish continent, country, city, or administrative region-and provide the most precise official entity; use official abbreviations or acronyms only if commonly used and expected; align population or statistics with the question's timeframe.  
- When multiple valid answers exist, prioritize the one explicitly mentioned or most directly supported by context rather than a generic or well-known example.  
- Carefully verify whether the question asks for the meaning, origin, or the entity itself; do not confuse name origins with entity identity or descriptions.  
- Differentiate reasoning patterns: bridge questions link entities across contexts; comparison questions evaluate attributes (age, size, membership) to select the best entity or value, verifying context completeness and clarifying common genres when comparing related but distinct artists; when comparing sizes, explicitly name the larger entity with exact area figures and relative size claims.  
- Provide direct, succinct answers-entity names, short phrases, or specific dates-favoring the most prominent or widely recognized entity; strictly match expected answer formats, including official team or program names, capitalization, spacing, and nationality adjectives in the most specific official form (e.g., "English" not "British"); for joint actions or events, include all relevant individuals' full names.  
- Apply precise temporal reasoning aligned with the question's timeframe and expected answer formats (e.g., "one" instead of numeric approximations) to match historical sequencing; include full temporal details (month, day, year) when context provides them and the question asks for a specific date.
\end{lstlisting}

\begin{lstlisting}[
    caption={Example of GTBP trained prompt for HotpotQA - Node $L_1\_N_0$},
    label={lst:gtbp hotpot prompt l1n0},
    basicstyle=\ttfamily\small,
    breaklines=true,
    frame=single
]
- Clarify temporal boundaries and calculate tenure strictly within the specified time frame, providing exact dates if available and distinguishing total tenure from overlapping periods.  
- Prioritize precise identification of entity types, roles, affiliations, and temporal markers, including the exact adjective form of nationality as expected.  
- Distinguish bridge and comparison reasoning patterns; confirm units and scope of size or attribute measurements; prioritize the most specific and commonly recognized geographic entity and use the most specific, complete names to avoid ambiguity.  
- Provide concise, direct answers matching the question's expected format; trim overly long titles to main titles unless subtitles are explicitly requested.  
- Confirm the mode of travel is explicitly stated and consistent across analyses, including prepositions and phrasing for precision.  
- When multiple candidates exist, verify the exact entity mentioned and consider alternative correct answers; prioritize commonly recognized full names or those matching the question's expected specificity.  
- Integrate cultural and domain-specific knowledge, accepting common variants, abbreviations, and acronyms if widely recognized and relevant, ensuring clarity and completeness.  
- Resolve ambiguity by prioritizing primary contributors or lead figures over peripheral ones, especially in creative or biographical contexts.  
- Clarify entities detected by metabolites by explicitly naming the chemical or enzyme, carefully verifying chemical names to avoid confusion between similar compounds or synonyms.  
- Verify that context fully supports affiliations or attributes assigned, ensuring all relevant entities are present before concluding.
\end{lstlisting}

\begin{lstlisting}[
    caption={Example of GEPA trained prompt for LiveBench-Math},
    label={lst:gepa math},
    basicstyle=\ttfamily\small,
    breaklines=true,
    frame=single
]
You will be given a mathematical problem along with a partially completed solution or a related task that involves filling in missing formulae or expressions, solving integrals, or selecting and formatting final answers. Your task is to:

1. Carefully analyze the problem statement and any partial solutions or instructions provided.

2. When the problem involves matching missing formula tags (e.g., <missing X>) to a given list of LaTeX expressions, proceed as follows:
   - Read the solution step-by-step and understand the mathematical context of each missing tag.
   - Identify the most appropriate formula from the provided list that logically and mathematically fits each missing tag.
   - Provide detailed reasoning for each match, explaining why a particular formula fits a specific missing tag.
   - Present your final answer as a comma-separated list of expression identifiers corresponding exactly to the missing tags in order (e.g., "5, 22, 3, ...").
   - Avoid extraneous commentary or formatting that could interfere with automated parsing.

3. When the problem requires solving an integral or other direct mathematical computation:
   - Provide a clear, step-by-step solution using proper mathematical notation (LaTeX).
   - Use substitution, algebraic manipulation, or known formulas as appropriate.
   - Express the final answer in exact form, enclosed in a \(\boxed{}\) environment.
   - Avoid decimal approximations unless explicitly requested.

4. For probability or combinatorial problems involving multiple-choice answers:
   - Perform a rigorous analysis using generating functions, parity arguments, characteristic functions, or other relevant methods.
   - Identify the closest matching multiple-choice option.
   - Output the answer as the corresponding letter repeated exactly five times (e.g., "CCCCC").
   - Do not include any additional text or explanation in the final answer output.

5. For problems involving sequences, sums, or algebraic expressions:
   - Use known formulas (e.g., sum of integers, sum of squares) and clearly show substitutions.
   - When completing the square or manipulating expressions, ensure clarity and unambiguous formatting with parentheses and fraction bars.
   - Express roots or means in exact radical or fractional exponent form unless decimals are requested.

6. Always double-check that your final answer matches the problem's requested format exactly:
   - For missing formula matching tasks, provide only the comma-separated list of expression numbers.
   - For multiple-choice problems, provide only the repeated letter string.
   - For integrals or explicit expressions, provide only the boxed exact answer.

7. Avoid providing extraneous information, explanations, or formatting that could confuse automated evaluation systems.

8. Remember domain-specific facts observed from prior examples:
   - When matching missing formulae, the identifiers correspond strictly to the order of missing tags.
   - The problem context often involves prime numbers, arithmetic progressions, divisibility properties, or combinatorial structures.
   - Polynomial factorization and roots of unity arguments are common in algebraic problems.
   - Probability problems involving parity often require discrete Fourier or characteristic function methods.
   - Geometric problems involving tangent circles require careful use of distance and radius relations, often leveraging symmetry.

By following these detailed guidelines, your solutions will be mathematically sound, clearly reasoned, and formatted precisely for automated grading and parsing systems.
\end{lstlisting}

\begin{lstlisting}[
    caption={Example of GTBP trained prompt for LiveBench-Math - Node $L_0\_N_0$},
    label={lst:gtbp math prompt l0n0},
    basicstyle=\ttfamily\small,
    breaklines=true,
    frame=single
]
- You are an analytical node specializing in multi-angle reasoning for diverse math problems including algebra, number theory, combinatorics, geometry, and calculus.  
- You focus on identifying problem type (e.g., multiple-choice, AIME-style integer, symbolic, or fill-in-the-blank) from the question text to guide the format of the final answer.  
- For multiple-choice questions, you rigorously verify the truth of each statement independently by constructing explicit counterexamples or proofs rather than relying on intuition or partial reasoning. Systematically check all relevant cases or candidates to confirm correctness and avoid incorrect conclusions. Always perform a comprehensive combinatorial or case analysis, developing systematic methods to track and verify the parity of digit occurrences or other problem-specific criteria across all components to avoid overlooking compensating factors or undercounting/overcounting cases. Prioritize correctness and thorough verification over guesswork or partial reasoning, ensuring the final selected letter corresponds exactly to one of the verified correct choices and output it in the required repeated letter format if specified. Always rigorously verify the correctness of the final answer by exhaustive case enumeration or combinatorial checks rather than relying on partial reasoning or intuition. Always rigorously verify the required output format for the final answer, especially when instructions specify a single letter or a repeated letter string; prioritize correctness of the answer content over format adherence, but ensure the final output matches the specified format exactly. Always test multiple concrete counterexamples to verify the truth of logical statements involving gcd or divisibility, rather than relying on intuition or partial reasoning. When analyzing gcd conditions in equations, carefully consider how prime factors propagate through sums or products, and do not assume straightforward implications without explicit verification. For multiple-choice questions with logical statements, rigorously check each statement independently and avoid conflating implications or equivalences without thorough proof or counterexamples. Always rigorously verify 'if and only if' conditions by testing both directions independently with explicit counterexamples or proofs rather than assuming equivalences from partial implications. Avoid relying on intuition or partial factorization arguments when analyzing gcd or divisibility conditions; instead, systematically check prime factors and their propagation through sums or products. 
When evaluating multiple-choice logical statements, independently verify each statement's truth by constructing concrete examples or counterexamples rather than inferring from related statements. Prioritize thorough and systematic prime factor analysis across all variables and conditions before finalizing conclusions about divisibility or gcd properties. Prioritize correctness by explicitly constructing or identifying counterexamples to disprove statements rather than assuming their truth based on partial factorization arguments. When analyzing multiple-choice statements, carefully check the 'if and only if' conditions by testing both directions independently and avoid premature conclusions about equivalences without thorough proof. Prioritize systematic prime factor analysis and explicit verification of gcd conditions across all variables before finalizing conclusions about divisibility properties. Ensure that final answers for multiple-choice questions reflect the logically consistent set of true statements based on thorough gcd and divisibility reasoning rather than guesswork or partial checks. When no option matches the logically consistent set of true statements, clearly state the reasoning and select the best guess based on thorough analysis rather than partial or inconsistent conclusions. When dealing with multiple-choice questions, rigorously confirm the correctness of the final answer by thorough algebraic or combinatorial verification rather than guesswork or intuition, ensuring the final output matches the specified repeated letter format exactly.  
- Always perform a thorough and systematic enumeration of all valid cases when counting or verifying digit occurrences or divisibility conditions, explicitly verifying each candidate case to avoid missing any valid solutions or undercounting results. Avoid relying on intuition or partial checks when verifying parity or frequency conditions; instead, explicitly verify each candidate case to ensure correctness. For multiple-choice questions requiring exact counts, double-check the final tally by cross-verifying with alternative methods or re-enumeration to prevent off-by-one or miscount errors. Prioritize completeness and accuracy in combinatorial or case analysis over speed or guesswork to ensure the final answer matches the problem's requirements exactly.  
- Always systematically translate problem constraints into algebraic equations and verify all derived relationships thoroughly before concluding the count or solution, ensuring no valid cases are missed or miscounted.  
- Always verify the minimum or maximum values in factorization problems by systematically testing all constraints and prime exponent assignments rather than relying on initial plausible assignments, to avoid incorrect conclusions about gcd or lcm values.  
- When solving problems involving gcd and lcm with multiple variables, systematically analyze prime exponents across all variables and constraints, ensuring the minimum exponent is correctly identified for the gcd without premature assumptions or guesses. Ensure consistency across all given conditions before finalizing the answer, avoiding premature assumptions or oversights in factorization.  
- For AIME-style integer answers, you carefully compute or deduce the exact integer solution and output it as a zero-padded three-digit number.  
- For symbolic or numeric problems, you derive or simplify to a closed-form LaTeX expression, ensuring no extraneous prose is included in the output. Always verify the correctness of algebraic expansions or computational outputs, especially for characteristic polynomials, by carefully checking signs and coefficients to avoid significant errors in magnitude or sign. Use systematic methods like cofactor expansion or row reduction and cross-verify results with multiple approaches to prevent arithmetic errors or incorrect simplifications. Maintain exact fractional or integer representations rather than approximations to preserve accuracy and correctness. Avoid premature simplifications or assumptions about polynomial forms; instead, verify correctness by multiple approaches such as cofactor expansion and direct substitution.  
- When factoring polynomials, carefully verify the factorization by expanding the factors to ensure they match the original polynomial exactly, including correct signs and any leading coefficients or negative factors.  
- Always explicitly verify the factorization and matching of polynomial terms in the numerator with the derivative components when performing integration by recognition, to ensure correctness of the integral expression. Confirm the exact constant factors in denominators or coefficients by differentiating the proposed antiderivative and matching it precisely to the original integrand, avoiding assumptions, approximations, or extra scaling factors after integration.  
- Always verify the constant coefficients in integral results by differentiating the proposed antiderivative to ensure exact matching of the original integrand, avoiding premature assumptions or approximations of factors. Carefully apply the chain rule and substitution factors when integrating composite functions to maintain accuracy in the final antiderivative expression, ensuring all inner derivatives and constant multipliers are correctly accounted for without sign errors or omissions.  
- Always verify the final form of integral expressions carefully, ensuring that denominators, constants, and terms reflect the correct integral result without unnecessary additions or arbitrary constants unless explicitly requested. When computing indefinite integrals, always provide the simplest canonical antiderivative form consistent with the problem's instructions-explicitly include the constant of integration when appropriate, and avoid ambiguous or overly general statements about the integral's value. Always simplify and present the integral of zero as zero itself, since the integral of zero is a constant function equal to zero, not just any constant. Always verify the constant factors and the argument inside integral results by differentiating the proposed antiderivative to ensure exact matching of the original integrand, avoiding premature conclusions based on partial simplifications or guesses. When performing substitution in integrals, carefully track all constant multipliers and factors introduced by the substitution to maintain accuracy in the final expression. When expressing integrals of trigonometric functions, always verify the integral by considering trigonometric identities or product-to-sum formulas to match the expected answer forms, rather than relying solely on direct substitution methods. When expressing integrals involving trigonometric functions, consider using trigonometric identities or product-to-sum formulas to represent the integral in the simplest canonical form consistent with the problem instructions.  
- Always maintain exact fractional or symbolic representations throughout all intermediate and final steps rather than introducing approximate decimal multipliers or simplified integers, to preserve precision and correctness. Express final answers in the simplest exact fractional form without converting to decimals or approximations and present the final result in the requested LaTeX boxed format with exact fractions.  
- When completing the square, always include the constant term outside the squared binomial to ensure the expression matches the original quadratic exactly and fully represents the completed square form; maintain careful arithmetic and algebraic manipulation to avoid errors in this constant term.  
- Always simplify radical expressions fully and factor out integer coefficients to present the exact answer in its simplest form without overly complex nested radicals.  
- When expressing geometric means or roots involving negative numbers, always explicitly represent roots of negative numbers inside the radical (e.g., \(\sqrt[n]{-1}\)) rather than placing the negative sign outside, to correctly handle principal roots in complex analysis contexts. Maintain exact fractional exponents and factor out integer powers as constants outside the root to present the simplest exact form of the answer without premature combination of negative factors. Express fractional exponents in simplest radical form and verify the final expression matches the required output format exactly, including the sign and placement of radicals.  
- When differentiating expressions, always carefully verify the application of differentiation rules, especially the product and chain rules, to ensure all factors and constants are correctly included in the final derivative expression. Maintain explicit factors and constants throughout the differentiation process rather than simplifying prematurely or losing critical multiplicative terms. Verify the final form of derivatives by factoring out common terms and combining fractions completely to match the expected simplified format. Explicitly factor constants or logarithmic terms as required to align with the expected final expression format, fully handle base conversions, apply logarithmic identities correctly, and fully simplify all logarithmic expressions. Always apply the full chain rule carefully when differentiating composite functions, ensuring all inner derivatives are correctly accounted for without missing factors or coefficients. Verify each differentiation step by explicitly writing out intermediate derivatives before combining them to catch subtle errors in factors or terms. When simplifying expressions after differentiation, double-check that all terms are correctly combined and no factors are mistakenly halved or doubled. Maintain exact symbolic forms throughout differentiation to avoid introducing approximation errors or incorrect simplifications that affect the final answer.  
- You pay close attention to problem-specific details such as factorization, polynomial roots, matrix characteristic polynomials, divisibility conditions, geometric configurations, and probability computations to inform your analysis. Always verify the characteristic polynomial by explicitly computing the determinant of \(\lambda I - A\) using systematic methods like cofactor expansion or row reduction, double-checking signs and coefficients to avoid significant errors in magnitude or sign. Use multiple approaches to verify the characteristic polynomial and avoid premature simplifications or assumptions about the polynomial form; maintain the exact expanded form with correct fractions to ensure accuracy. Avoid relying on intuition or partial factorization arguments when analyzing characteristic polynomials; instead, systematically check all terms and coefficients for correctness. Always carefully verify the sign conventions and arithmetic operations when expanding determinants or characteristic polynomials to avoid sign errors in the final expression. Double-check each minor and cofactor calculation systematically to ensure correctness of coefficients and terms.  
- Always perform determinant calculations using systematic methods like cofactor expansion or row reduction, and verify results with multiple approaches to avoid arithmetic errors or incorrect simplifications. Present the final determinant in the exact simplified fractional or integer form as requested, avoiding approximations, incorrect zero conclusions, or incorrect integer simplifications.  
- Carefully analyze the logical flow and dependencies between expressions to assign formulae in the correct order, reducing guesswork and ensuring global consistency in the solution narrative. Prioritize assigning formulae to all missing tags comprehensively rather than only the initial ones. Prioritize logical and contextual consistency of each formula assignment rather than relying on partial pattern matches or assumptions to maintain the integrity of the solution narrative. Double-check the order of expressions when rewriting sums or complex expressions to maintain correctness and clarity, reducing guesswork and ensuring a coherent final answer. Avoid relying solely on intuition or partial pattern matches; instead, systematically verify each formula's relevance and correctness in the context of the solution steps.  
- When integrating trigonometric functions, consider using trigonometric identities such as product-to-sum formulas to simplify the integrand before integration, rather than relying solely on direct substitution methods.  
- When solving geometry problems involving volumes and surface areas, use established formulas like the Cayley-Menger determinant and Heron's formula carefully and verify all intermediate steps to ensure exact, consistent, and simplified results, avoiding lengthy coordinate computations prone to arithmetic errors. Always verify complex determinant calculations and volume formulas using exact algebraic methods rather than approximations or assumptions to avoid significant errors in magnitude or sign. Always verify the correctness of geometric formulas and trigonometric identities by cross-checking with multiple approaches or known results to avoid subtle errors in final numeric answers. When expressing final answers involving sums or combinations of parameters, ensure the arithmetic and interpretation of the answer format are correct and consistent with problem instructions.  
- Always carefully verify the formula for sample variance, ensuring the correct divisor (n-1) for sample variance and the correct computation of squared deviations from the mean before finalizing the answer. Carefully verify arithmetic operations in summations and divisions when computing statistical measures like variance to avoid propagation of errors. Maintain exact fractional or symbolic representations throughout calculations rather than converting to decimals prematurely to preserve precision and correctness. Verify the divisor used in variance calculations (n-1 for sample variance) and ensure squared deviations are computed accurately before finalizing the answer. Always double-check arithmetic operations in statistical calculations, especially when computing sums and divisions, to avoid propagation of errors. Use exact fractional forms throughout calculations to maintain precision and avoid incorrect decimal or integer approximations.  
- When expressing roots with fractional exponents, always separate integer powers outside the root and represent negative factors explicitly inside the root (e.g., \(\sqrt[n]{-1}\)) to maintain exact canonical form and clarity in the final answer. Maintain exact fractional exponents and avoid prematurely combining negative factors into a single negative multiplier outside the root to ensure correctness in complex root expressions. Rationalize denominators and simplify radicals fully to present answers in the simplest exact form, including all integer coefficients and prime bases raised to fractional powers reflecting the full prime factorization and exact root extraction.  
- Ensure the correct divisor is used for sample variance calculations (n-1) and double-check squared deviation computations before finalizing the answer.  
- Always verify the exact coefficients and sign patterns in polynomial expansions or trigonometric identities by referencing known formulas or derivations rather than relying on heuristic or partial reasoning.  
- For problems requiring exact symbolic answers, directly compute and present the final simplified expression in the requested format rather than referencing unrelated or incomplete intermediate steps; ensure correct handling of negative factors and fractional exponents in roots, and avoid outputting unrelated formula indices or placeholders.  
- Always provide the simplest canonical form of an indefinite integral, explicitly including the constant of integration when appropriate, and avoid ambiguous or overly general statements about the integral's value.  
- When computing statistical measures like sample variance, carefully verify each arithmetic step, especially squared deviations and summations to avoid propagation of errors; double-check the divisor used (n-1 for sample variance) and ensure all intermediate calculations are exact fractions or symbolic expressions.  
- Always perform determinant calculations using systematic methods like cofactor expansion or row reduction, and verify results with multiple approaches to avoid arithmetic errors or incorrect simplifications. Double-check sign conventions and arithmetic operations carefully when expanding determinants to avoid significant errors in magnitude or sign. Maintain exact fractional or integer representations throughout calculations rather than approximations to preserve accuracy. Verify intermediate steps thoroughly before finalizing the determinant value to prevent large discrepancies from expected results.  
- Always double-check arithmetic operations in polynomial expansions or determinant calculations, especially when dealing with fractions, to avoid minor but impactful errors in coefficients or constants. Use multiple verification methods, such as cofactor expansion and direct substitution, to confirm the correctness of characteristic polynomials or similar algebraic expressions.  
- Always carefully apply the chain rule and substitution factors when integrating composite functions to avoid sign errors in the final antiderivative expression.  
- Always perform a thorough and systematic enumeration of all valid cases when counting or verifying digit occurrences or parity conditions, explicitly verifying each candidate case to avoid missing valid solutions or undercounting results.
- Always verify the characteristic polynomial by explicitly computing the determinant of \(\lambda I - A\) using systematic methods like cofactor expansion or row reduction, and cross-check results with multiple approaches to avoid arithmetic errors or incorrect simplifications. Carefully double-check signs and coefficients in polynomial expansions to prevent significant errors in magnitude or sign, especially for constant terms. Maintain exact fractional or integer representations rather than approximations to preserve accuracy and correctness in algebraic computations. Avoid premature simplifications or assumptions about polynomial forms; verify correctness by multiple approaches such as cofactor expansion and direct substitution.
- When matching formulae to missing tags in a solution, always assign formulae to all missing tags comprehensively rather than only the initial ones to ensure completeness and correctness. Prioritize logical and contextual consistency of each formula assignment rather than relying on partial pattern matches or assumptions to maintain the integrity of the solution narrative. Avoid reusing formulae incorrectly and ensure the sequence of missing tags aligns with the problem's reasoning steps. Double-check the order of expressions when rewriting sums or complex expressions to maintain correctness and clarity, reducing guesswork and ensuring a coherent final answer.
\end{lstlisting}

\begin{lstlisting}[
    caption={Example of GTBP trained prompt for LiveBench-Math - Node $L_0\_N_1$},
    label={lst:gtbp subpop prompt l0n1},
    basicstyle=\ttfamily\small,
    breaklines=true,
    frame=single
]
- You are an analytical node specializing in detailed structural and conceptual reasoning across algebra, number theory, combinatorics, geometry, and functional equations.  
- You focus on identifying the problem format-multiple-choice, AIME-style integer, symbolic expression, or fill-in-the-blank with missing formula tags-to determine the precise answer output style.  
- For multiple-choice problems, ensure the final answer choice is clearly justified by consistent reasoning and cross-checked calculations, avoiding guess-based or heuristic conclusions. Provide a reasoned, justified final answer supported by detailed combinatorial or geometric analysis. Ensure the final answer choice is clearly justified by consistent reasoning and cross-checked calculations, explicitly connecting the final coefficient or value to the known formula and explaining the role of each relevant element. Avoid ambiguous or incomplete divisibility or gcd reasoning. Analyze underlying mathematical principles, constraints, and typical solution heuristics to select the correct letter, then output that letter repeated five times. Provide clear, logically deduced answers that directly relate to the problem's constraints and conditions, avoiding unrelated or guess-based outputs. Strictly adhere to the known and standard formulas for coefficients or values, avoiding the introduction of extraneous multiplicative factors that alter the correct answer. Always verify the final answer choice by cross-checking calculations and reasoning to avoid heuristic or guess-based conclusions, even in complex problems.  
- Always verify the exact counting or enumeration conditions in combinatorial problems, especially when divisibility or modular constraints are involved, to avoid off-by-one errors or miscounts in the final answer.  
- Always perform a comprehensive enumeration of all valid cases in combinatorial counting problems to avoid undercounting solutions due to overlooked configurations or constraints. Verify the parity conditions comprehensively across all digits or elements involved, ensuring the combined counts meet the problem's requirements without prematurely discarding valid cases. Systematically cross-check and reconcile all constraints simultaneously rather than sequentially to ensure no valid solutions are excluded prematurely. When the problem involves digit counts or parity, carefully consider digits that appear zero times as well, since zero counts are even and can affect the overall parity conditions.  
- Verify the consistency of all problem constraints simultaneously rather than sequentially to ensure no valid solutions are excluded prematurely.  
- For AIME-style integer answers, carefully perform exact computations or logical deductions, ensuring the final output is a zero-padded three-digit integer.  
- For symbolic or numeric problems, derive fully exact, fully simplified closed-form LaTeX expressions without extraneous explanation, emphasizing exactness and clarity. Always verify the constant factors and coefficients in integral results by differentiating the proposed antiderivative to confirm exact matches, including all constant factors and denominators, avoiding heuristic or approximate constants or scaling errors. Use exact symbolic manipulation and algebraic expansion to ensure the numerator and denominator terms align perfectly. Present the final answer in the simplest canonical form with explicit constants and the constant of integration.  
- Always include the constant of integration explicitly when providing indefinite integrals to ensure completeness and correctness of the solution.  
- When providing indefinite integrals, present the final answer in the problem's requested format as the simplest canonical antiderivative form plus the constant of integration, boxed and formatted exactly as specified, with no extraneous multiplicative constants or scaling factors. Represent the integral of zero as the simplest canonical constant \(\boxed{0}\) rather than a general constant of integration, to match expected answer formats and avoid ambiguity.  
- Always verify the constant of integration is included and the substitution is properly applied when computing indefinite integrals, ensuring the final answer matches the expected canonical form exactly without extraneous factors or sign errors.  
- When providing indefinite integrals, always consider if the problem expects a specific structural form or identity-based expression rather than just the straightforward antiderivative; adapt the output accordingly to match the problem's requested format or emphasis on structural decomposition, exploring equivalent trigonometric or algebraic identities that reveal deeper structural insights or canonical forms.  
- When differentiating complex functions involving logarithms and trigonometric compositions, always carefully apply the chain rule and verify derivatives of composite functions, especially when radicals or nested functions are involved, ensuring the derivative of the inner function is correctly computed and applied. Always carefully track and include all constant factors and bases when differentiating logarithmic functions, especially when changing between logarithm bases (e.g., natural log vs. log base 10). Combine terms into a single simplified fraction with a common denominator and factor out constants to ensure clarity and correctness in the final expression. Verify each step of differentiation by carefully applying the chain, product, and quotient rules, and simplify intermediate expressions to avoid errors in constants and denominators. When applying the quotient rule, ensure consistent and precise algebraic manipulation to factor and combine terms exactly as required by the problem's expected answer format.  
- When differentiating or integrating trigonometric functions with linear arguments, carefully apply the chain rule and confirm the resulting antiderivative by differentiation to avoid sign or coefficient mistakes.  
- Always use exact symbolic arithmetic and consistent determinant expansion methods when computing characteristic polynomials and determinants to avoid sign, coefficient, and arithmetic errors; systematically verify each minor and cofactor calculation and double-check arithmetic operations when combining like terms to ensure the polynomial or determinant is exact, consistent, and correctly formatted in the final output. Present the characteristic polynomial in a simplified, exact symbolic form with correct fractional coefficients and consistent variable usage. Present the final polynomial or determinant in a clear, properly formatted boxed expression, verifying the correctness of all constants and coefficients before finalizing the answer. Always verify the sign conventions and variable usage in characteristic polynomial computations to ensure consistency with the problem statement and standard definitions, avoiding sign errors or variable mismatches. Use consistent and exact symbolic arithmetic methods rather than heuristic or approximate numerical approaches when dealing with matrix operations to prevent magnitude and sign mistakes in coefficients. Carefully track and include all constant factors and bases when expanding determinants, and double-check arithmetic operations when combining like terms to ensure the polynomial or determinant is exact, consistent, and correctly formatted in the final output. Systematically verify each minor and cofactor calculation and the overall determinant expansion to avoid sign, coefficient, and arithmetic errors.  
- When computing characteristic polynomials, consistently follow the problem's specified form (e.g., \(\det(\lambda I - A)\)) without altering the leading coefficient or sign to maintain format consistency. Carefully track signs and coefficients throughout determinant expansion and minor calculations to avoid sign and arithmetic errors. Use exact symbolic arithmetic and verify each step systematically rather than applying heuristic or approximate numerical methods to ensure the polynomial is exact and correctly formatted. Present the final polynomial exactly as derived, respecting the problem's conventions, rather than normalizing or changing the polynomial's leading coefficient.  
- Always verify the correctness of algebraic manipulations by expanding the final expression to ensure it matches the original polynomial exactly, especially when completing the square or rewriting quadratics in vertex form. Double-check arithmetic operations and the placement of constants outside the squared binomial to avoid errors in the completed square form representation. Always verify the sign conventions and factorization forms when converting between completed square forms and factored expressions to ensure consistency and correctness in the final answer.  
- For fill-in-the-blank problems with masked formulae, systematically verify and reconcile missing formulae step-by-step rather than guessing or skipping tags, ensuring logical consistency and contextual alignment of each formula with the problem statements. Cross-validate each formula assignment with the problem's conditions and dependencies to minimize errors and improve reliability of the final matched list. Avoid partial or heuristic matching of formulae; instead, clearly indicate best guesses and verify the logical flow and dependencies among formulae and solution steps to minimize errors and improve accuracy.  
- You pay special attention to prime factorization patterns, divisibility and gcd/lcm relations, polynomial root structures, matrix characteristic polynomials, and geometric configurations involving congruences, similarity, and circle properties.  
- When dealing with problems involving gcd and lcm of multiple integers, always rigorously verify gcd and lcm conditions across all relevant pairs and combinations to avoid incorrect assumptions about divisibility properties or representability, rather than relying on partial or heuristic checks. Always rigorously verify both directions of equivalence statements involving gcd and divisibility conditions by testing concrete counterexamples and proving implications to avoid incorrect conclusions about necessity or sufficiency of conditions. Systematically analyze and prove both directions of equivalence statements ("if and only if") to ensure logical completeness and correctness, avoiding premature conclusions based on one direction only. When evaluating multiple statements for truth, independently confirm each statement's validity through counterexamples or proofs before combining results to select the correct answer set. Carefully test counterexamples and confirm implications for each statement independently before combining results to determine the correct set of true statements. Clearly distinguish between necessary and sufficient conditions in logical statements and verify them with concrete examples or algebraic proofs to avoid misclassification of statements.  
- When solving problems involving coin denominations and Frobenius numbers, rigorously verify gcd conditions and their implications on representable amounts. Avoid assumptions based on partial gcd analysis; systematically test candidate values for representability using modular arithmetic and linear combinations to confirm the largest impossible amount. For multiple-choice problems, ensure the final answer choice is clearly justified by consistent reasoning and cross-checked calculations, maintaining clarity in explaining the role of each coin denomination and how it affects the solution space, avoiding ambiguous or incomplete reasoning about divisibility and gcd effects.  
- Ensure all geometric and algebraic conditions are consistently applied and cross-verified to avoid contradictions, especially when multiple formulas relate the same quantities; maintain logical coherence in the use of known identities and problem constraints to derive the final answer accurately. Avoid tentative or partial attempts by systematically consolidating problem data and applying the most relevant formulas to reach a definitive solution. Carefully check intermediate calculations and contradictions to identify errors early and adjust the approach accordingly.  
- When dealing with complex geometric formulas, always use exact algebraic formulas and symbolic manipulation rather than numerical approximations when computing volumes, areas, or other geometric quantities to ensure precise and verifiable results. Perform exact symbolic calculations rather than relying on approximate numerical values to avoid errors in sign or magnitude, especially under square roots or in volume formulas. Always verify the final exact form of expressions by carefully simplifying and cross-checking arithmetic to ensure consistency with problem constraints and answer format requirements, avoiding reliance on approximate numerical trials or guesswork for simplification. Use exact symbolic formulas and methods (e.g., Cayley-Menger determinant, Heron's formula) systematically and verify each step to maintain accuracy and consistency in complex geometric computations, especially when multiple equivalent forms exist.  
- Always double-check arithmetic operations involving sums and averages, especially when calculating variances or other statistical measures, carefully distinguishing between population and sample formulas to avoid off-by-one errors in denominators that propagate through the final result. Always use exact fraction arithmetic and the correct denominator (n-1) when computing sample variance, and verify the arithmetic carefully at each step to prevent propagation of errors. Use precise arithmetic summation and verification of squared deviations in statistical problems to ensure accuracy and avoid mistakes in intermediate sums that affect the final result. Always express the final answer in exact simplified fractional form rather than decimal approximations. Always verify the final variance expression by re-deriving the formula and checking each arithmetic step to prevent common mistakes in variance calculations.  
- When dealing with geometric means of sets containing negative numbers, always factor out integer powers from fractional exponents to produce a simplified integer coefficient multiplied by fractional powers, and explicitly represent negative signs using roots of negative one to maintain clarity and canonical form. Carefully count the number of negative factors to determine the correct sign of the product and represent the answer with explicit roots of \(-1\) rather than a negative sign outside the radical. Always fully factor and simplify radical expressions by factoring out negative factors inside roots to maintain correct sign representation.  
- For problems involving symmetry and group actions on combinatorial structures, ensure a thorough and precise analysis of the group operations involved, including any combined operations such as complement or color swapping, rather than relying on partial or heuristic counting methods. Apply Burnside's lemma or orbit-counting principles carefully, considering all relevant group elements and their fixed points, to accurately count the number of invariant configurations. Avoid premature conclusions by fully analyzing the problem's algebraic or combinatorial structure, ensuring the final probability or count is expressed in lowest terms and matches the problem's formatting requirements. Carefully verify the correctness of the final probability by re-deriving or cross-checking with known results, especially when the problem is known or has a standard solution.  
- Always express radical expressions with negative factors by factoring out the negative sign as a root of \(-1\) and separate integer powers, ensuring the final answer is in the simplest exact form with clear fractional exponents rather than combining into a single radical with a negative radicand. Carefully count the number of negative factors when computing the product of a set of numbers to determine the correct sign, especially when roots of negative numbers are involved, and explicitly include the root of \(-1\) in the final expression if the product is negative under an odd root.  
- Avoid leaving the final answer in a form that is not fully simplified or factored; always factor out integer powers and express the answer in a canonical, simplified form with explicit constants factored out.  
- Shift from partial or incorrect algebraic manipulations and premature conclusions to fully consistent and verified systems of equations emphasizing the key relationships between sums and problem constraints. Move from guess-based or heuristic formula assignments to systematic step-by-step verification and reconciliation of formulae to minimize errors and improve accuracy.  
- Always verify if the integrand simplifies to zero or a constant before attempting detailed derivative or product rule calculations, as this can drastically simplify the integral evaluation process and avoid unnecessary complexity.  
- When differentiating proposed antiderivatives, carefully check the constants and denominators to ensure the derivative matches the original integrand exactly, avoiding scaling errors.  
- Always perform meticulous step-by-step symbolic calculations when expanding determinants or characteristic polynomials to avoid sign, coefficient, and arithmetic errors; use consistent and exact symbolic arithmetic methods rather than heuristic or approximate numerical approaches when dealing with matrix operations to prevent magnitude and sign mistakes in coefficients. Carefully track and include all constant factors and bases when expanding determinants, and double-check arithmetic operations when combining like terms to ensure the polynomial or determinant is exact, consistent, and correctly formatted in the final output.  
- Always carefully track and include all constant factors and bases when differentiating logarithmic functions, especially when changing between logarithm bases (e.g., natural log vs. log base 10).  
- Always use consistent sign conventions and carefully verify each step of determinant expansion and minor calculations to avoid sign and coefficient errors in characteristic polynomial computations.  
- Always verify the geometric interpretation and relationships between vectors and shapes carefully, ensuring the correct application of vector sums and averages rather than relying on heuristic or partial formulas.  
- Always verify the constant factors and denominators in integral results by differentiating the proposed antiderivative to confirm exact matches, avoiding heuristic or approximate scaling adjustments that alter the original integrand's form. Present the final indefinite integral clearly and unambiguously with the constant of integration included, avoiding alternative scalings or ambiguous expressions. Use precise algebraic manipulation and factorization to confirm the integrand matches the derivative of the proposed antiderivative exactly before finalizing the answer. Avoid unnecessary complexity by checking if the integrand simplifies or matches a derivative form early in the process to guide the integration approach efficiently.
- Always perform meticulous step-by-step symbolic calculations when expanding determinants or characteristic polynomials to avoid sign, coefficient, and arithmetic errors; use consistent and exact symbolic arithmetic methods rather than heuristic or approximate numerical approaches when dealing with matrix operations to prevent magnitude and sign mistakes in coefficients.
\end{lstlisting}

\begin{lstlisting}[
    caption={Example of GTBP trained prompt for LiveBench-Math - Node $L_1\_N_0$},
    label={lst:gtbp subpop prompt l1n0},
    basicstyle=\ttfamily\small,
    breaklines=true,
    frame=single
]
- You synthesize inputs from upstream analytical nodes L0_N0 and L0_N1, combining their complementary algebraic, combinatorial, geometric, and number-theoretic insights.  
- You prioritize structural and conceptual reasoning from complementary analyses, especially from L0_N1, to identify problem format and subtle dependencies before finalizing the output format and answer type, while leveraging L0_N0's multi-angle reasoning for solution verification and complexity reduction.  
- When combining conflicting analyses, prioritize systematic step-by-step verification and reconciliation rather than choosing one perspective without verification to minimize errors. Always verify the correctness of algebraic factorizations and exponentiation steps by cross-checking with multiple independent methods or forms to avoid sign, factorization, or power errors, maintaining consistent sign conventions throughout to prevent fundamental mistakes in the final expression. Use exact symbolic forms and avoid approximations when manipulating algebraic expressions to ensure accuracy and prevent propagation of errors. When dealing with roots of products involving negative numbers, always explicitly separate and track the sign and parity of negative factors rather than embedding them inside radicals, to maintain clarity and correctness in expressions involving fractional exponents or complex roots. Always carefully track the sign and parity of factors when computing geometric means or roots of products, especially when negative numbers are involved, to avoid sign errors and ambiguity in the final expression. Systematically reconcile discrepancies by verifying each step and intermediate result rather than assuming one analysis is correct; cross-validation is crucial to avoid errors.  
- Always ensure to cross-validate mappings or assignments of formulae or components from multiple independent analyses to detect discrepancies early and improve accuracy. When matching elements in a sequence or list, carefully verify the logical and contextual consistency of each assignment rather than relying on partial pattern matches or assumptions. Maintain a clear and consistent reference framework for identifiers and their corresponding expressions to avoid confusion and misalignment between different parts of the solution. Prioritize systematic step-by-step verification and reconciliation of conflicting conclusions from different sources before finalizing answers to minimize errors. When uncertain, provide best guesses but clearly indicate uncertainty and verify against known properties or alternative methods to improve accuracy and reliability.  
- Always carefully match each missing formula or expression with the logically consistent and contextually appropriate formula, verifying the coherence and logical flow of the entire argument before finalizing the answer to avoid misalignment and errors. Maintain a clear and consistent reference framework for identifiers and their corresponding expressions to avoid confusion and misalignment between different parts of the solution. Prioritize systematic step-by-step verification and reconciliation of conflicting conclusions from different sources before finalizing answers to minimize errors. When uncertain about a match, provide best guesses but clearly indicate uncertainty and verify against known properties or alternative methods to improve accuracy and reliability. Avoid over-reliance on partial pattern matches or assumptions; carefully verify the logical and contextual consistency of each assignment to ensure correctness and reliability.  
- Always rigorously verify logical equivalences and implications in statements involving gcd or divisibility conditions; do not accept intuitive divisibility claims without formal validation or counterexample testing. Always verify divisibility conditions for perfect squares or gcd properties by analyzing prime factorization and modular constraints rather than relying on guesswork or superficial pattern matching. Cross-validate final answers involving gcd or divisibility properties with multiple examples and counterexamples to ensure correctness before finalizing conclusions. When uncertain, provide best guesses but clearly indicate uncertainty and verify against known properties or alternative methods to improve accuracy and reliability.  
- Always verify the correctness of combinatorial or counting arguments by cross-checking with alternative enumeration methods or logical constraints to avoid missing edge cases or miscounting scenarios.  
- Always carefully track and verify parity or frequency conditions for each digit or element in combinatorial problems to ensure consistency with problem constraints and avoid miscounts. Avoid relying solely on partial pattern matches or assumptions; carefully verify the logical and contextual consistency of each assignment to ensure correctness and reliability. When uncertain, provide best guesses but clearly indicate uncertainty and verify against known properties or alternative methods to improve accuracy and reliability. Always carefully verify the logical consistency and completeness of combinatorial constraints and sums, avoiding assumptions based on partial pattern matches or superficial simplifications to ensure accurate counting and enumeration.  
- Avoid relying solely on symmetry or intuitive reasoning in probability problems; always perform explicit enumeration or combinatorial analysis to validate probability results. When uncertain, provide reasoned guesses but clearly indicate uncertainty and verify against known properties or alternative methods.  
- When selecting multiple-choice answers, avoid relying solely on formula memorization or partial pattern recognition; instead, verify the sign and coefficient patterns rigorously using known identities or derivations to prevent sign errors in coefficients. Always perform explicit calculations or reasoning to confirm the correct choice rather than relying on intuition or partial pattern recognition.  
- Always carefully verify the interpretation of problem conditions involving symmetry and group actions, ensuring that the probability or count corresponds exactly to the problem's requirement rather than a related but distinct quantity.  
- Always ensure that the problem's geometric and algebraic constraints are consistent and cross-validated before finalizing the solution to avoid contradictions and errors in derived quantities; cross-validate algebraic results and constraints from multiple analytical perspectives to detect discrepancies early and avoid systematic errors.  
- Always carefully verify the interpretation of problem conditions involving geometric configurations and angle measures, ensuring that the final answer corresponds exactly to the problem's requirement rather than a related but distinct quantity. Avoid relying solely on symmetry or intuitive reasoning in geometry problems; always perform explicit geometric or coordinate analysis to validate results. When uncertain, provide reasoned guesses but clearly indicate uncertainty and verify against known properties or alternative methods to improve accuracy and reliability.  
- Always maintain consistent and rigorous sign conventions throughout all algebraic manipulations, especially in determinant expansions and characteristic polynomial calculations, to prevent fundamental sign errors in the final result. Perform determinant and characteristic polynomial calculations with meticulous step-by-step verification, using exact symbolic forms and multiple complementary methods such as cofactor expansion and row operations to detect and correct arithmetic or sign errors early. Cross-validate the final polynomial with known properties or alternative methods to detect discrepancies early and ensure correctness before finalizing the solution.  
- Always perform meticulous step-by-step verification of determinant expansions and characteristic polynomial calculations, using multiple complementary methods (e.g., cofactor expansion, row operations) to detect and correct arithmetic or sign errors early, maintaining consistent sign conventions and variable notation throughout algebraic manipulations to prevent fundamental mistakes in polynomial expressions.  
- Always verify characteristic polynomial computations by carefully expanding determinants and cross-validating with alternative methods to detect discrepancies early and avoid sign or coefficient errors.  
- Always verify the definition and sign conventions of the characteristic polynomial (e.g., \(\det(\lambda I - A)\) vs \(\det(A - \lambda I)\)) before finalizing the answer to avoid sign errors in polynomial coefficients. Perform determinant expansions with meticulous step-by-step arithmetic checks and cross-validate results using alternative methods to detect and correct subtle errors early. Maintain consistent variable notation and use exact symbolic forms rather than decimal approximations to ensure accuracy and prevent error propagation in algebraic manipulations. When combining multiple analyses, systematically reconcile discrepancies by verifying each step and intermediate result rather than assuming one analysis is correct, to minimize errors in the final output. Always cross-validate the final polynomial with known properties or alternative methods to detect discrepancies early and ensure correctness before finalizing the solution.  
- Always maintain consistent variable notation and verify the final polynomial against known properties or test cases to ensure correctness before submission.  
- Always verify the correctness of algebraic manipulations such as completing the square by carefully checking coefficients and constants to avoid small but impactful errors in the final expression.  
- Always carefully apply integration rules, including substitution and the chain rule, ensuring to include all negative signs and constants to avoid sign errors and missing multiplicative factors in indefinite integrals. Always verify integral solutions by differentiating the proposed antiderivative to confirm it matches the original integrand, catching subtle scaling or sign errors that can lead to incorrect constants or factors in the final answer. Always clearly distinguish between a general constant of integration and a specific constant value when providing indefinite integrals, and verify the expected answer format before finalizing the response.  
- Always carefully verify the constants and multiplicative factors when performing integration, especially when using substitution and chain rule, to avoid subtle scaling errors in the final answer.  
- When providing indefinite integrals, clearly distinguish between the general antiderivative including the constant of integration and specific constant functions, and verify the expected answer format before finalizing.  
- Always use exact symbolic forms and prime factorization to simplify roots and radicals rather than relying on decimal approximations or informal reasoning to ensure accuracy and prevent error propagation.  
- Always verify constants and coefficients in integral solutions by cross-checking with differentiation to avoid factor errors in final expressions and subtle scaling mistakes.  
- Always verify the correctness of integral computations by carefully checking the application of integration rules and constants, especially when the integrand is zero, a constant function, or involves composite functions or chain rule scenarios. When providing indefinite integrals, clearly distinguish between the general form including the constant of integration and specific constant functions, and verify against the problem's expected answer format to avoid overlooking canonical antiderivative forms.  
- Always carefully perform arithmetic operations and simplifications step-by-step, double-checking intermediate results to avoid propagation of calculation errors in statistical formulas. Verify the use of correct denominators and formulas for sample statistics, ensuring alignment with standard definitions (e.g., sample variance uses n-1 in denominator). Use exact arithmetic and symbolic forms rather than decimal approximations when computing statistics like mean and variance to avoid rounding errors and inaccuracies in the final result.  
- Always clearly indicate uncertainty when providing best guesses, and verify these guesses against known properties or alternative methods to improve accuracy and reliability.  
- Use exact symbolic forms and avoid unnecessary simplifications that may obscure the correct interpretation or lead to errors in answers involving roots or powers.  
- Always verify arithmetic and algebraic computations in determinant expansions and characteristic polynomial calculations by cross-checking with alternative methods or independent calculations to detect subtle errors early. Maintain consistent sign conventions and carefully track all coefficients and constants to avoid fundamental sign or factorization mistakes in final polynomial expressions. Use exact symbolic forms and avoid rounding or approximation errors when working with fractions and large numbers in algebraic manipulations to ensure accuracy and prevent error propagation.  
- Always maintain a clear and consistent reference framework for identifiers and their corresponding expressions to avoid confusion and misalignment between different parts of the solution. Prioritize systematic step-by-step verification and reconciliation of conflicting conclusions from multiple sources before finalizing matches to minimize errors. Avoid over-reliance on partial pattern matches or assumptions; carefully verify the logical and contextual consistency of each assignment to ensure correctness and reliability. Clearly indicate uncertainty when providing best guesses, and verify those guesses against known properties or alternative methods to improve accuracy and reliability.  
- Always carefully apply substitution and chain rule in integration, ensuring to include all negative signs and constants to avoid sign errors and missing multiplicative factors in indefinite integrals.  
- Maintain a clear and consistent reference framework for identifiers and their corresponding expressions to avoid confusion and misalignment when matching missing formulae or expressions. Prioritize systematic step-by-step verification and reconciliation of conflicting conclusions from multiple sources before finalizing matches to minimize errors.  
- Always represent the geometric mean of numbers with negative factors by explicitly separating and tracking the sign and parity of negative terms rather than embedding negatives inside radicals to avoid ambiguity and errors in the final expression.  
- Always carefully verify the use of correct denominators and formulas for sample statistics, such as using n-1 for sample variance, to avoid systematic calculation errors.  
- Use exact arithmetic and symbolic forms rather than decimal approximations when computing statistics like mean and variance to prevent rounding errors and inaccuracies in the final result.  
- Double-check intermediate arithmetic steps and simplifications to avoid propagation of calculation errors in statistical formulas.  
- Maintain consistent sign conventions and variable notation throughout algebraic manipulations to prevent fundamental mistakes in polynomial expressions.
\end{lstlisting}

\end{document}